\documentclass[12pt]{article}

\usepackage{etoolbox}
\newtoggle{colt}
\togglefalse{colt}

\iftoggle{colt}{}{
\usepackage[sort&compress,numbers]{natbib}
}
\usepackage{boxedminipage}
\usepackage{multirow,nicefrac}
\usepackage{makecell,upgreek}
\usepackage{footnote}
\usepackage{longtable}
\usepackage{tablefootnote}
\usepackage[T1]{fontenc}
\usepackage{verbatim} 
\usepackage[utf8]{inputenc}

\usepackage{booktabs}   
\usepackage{adjustbox}

\iftoggle{colt}{}{
\usepackage[margin=0.75in]{geometry}
\usepackage[ruled,vlined]{algorithm2e}
}

\usepackage[dvipsnames]{xcolor}

\usepackage{colortbl}
\definecolor{lightgray}{gray}{0.9}  %

\iftoggle{colt}{}{
\usepackage{amsmath}
\usepackage{amsthm}

\usepackage[breaklinks=true]{hyperref}
\hypersetup{
	colorlinks=true,
	linkcolor=blue,
	citecolor=blue,
	urlcolor=blue
}
}

\iftoggle{colt}{}{
  \usepackage[capitalise,nameinlink]{cleveref}
}
\usepackage{amsfonts,amssymb,mathrsfs}
\usepackage{mathtools}

\usepackage{appendix}

\iftoggle{colt}{
  
}{
\theoremstyle{plain}
	\newtheorem{theorem}{Theorem}

	\newtheorem{lemma}{Lemma}

	\newtheorem{proposition}{Proposition}

	\theoremstyle{definition}

	\newtheorem{definition}{Definition}

}

\Crefname{appendix}{Appendix}{Appendices}

\usepackage{autonum}

\usepackage[T1]{fontenc}
\usepackage[scaled]{beramono}
\usepackage{listings}
\definecolor{varorange}{RGB}{230,126,34}     %
\definecolor{opblue}{RGB}{52,152,219}        %
\definecolor{ifred}{RGB}{192,57,43}          %
\definecolor{commentgray}{RGB}{150,150,150}  %
\definecolor{codebg}{rgb}{0.98,0.98,0.98}    %

\usepackage{wrapfig}

\newsavebox\codebox

\usepackage{authblk}

\newcommand{\norm}[1]{\left\|#1\right\|}

\newcommand{\abs}[1]{\left|#1\right|}

\newcommand{\rr}{\mathbb{R}}
\newcommand{\ee}{\mathbb{E}}
\newcommand{\pp}{\mathbb{P}}

\newcommand{\kld}{\mathrm{D}_{\mathrm{KL}}}

\newcommand{\divf}[2]{\mathrm{D}_{f}\left( #1 \middle\| #2 \right)}
\newcommand{\dkl}[2]{\kld\left( #1 \middle\| #2 \right)}

\DeclareMathOperator*{\argmin}{argmin}

\DeclareMathOperator{\Cov}{Cov}

\DeclareMathOperator{\Med}{Med}

\newcommand{\covm}[3][M]{\mathrm{Cov}_{#1}\left( #2 \| #3 \right)}

\newcommand{\icov}[3][M]{\mathrm{ICov}_{#1}\left( #2 \| #3 \right)}

\newcommand{\ind}[1]{\mathbb{I}\left[ #1 \right]}

\newcommand{\jhat}{\widehat{j}}
\newcommand{\tv}{\mathrm{TV}}

\newcommand{\chisq}[2]{\mathrm{D}_{\chi^2}\left( #1 \middle\| #2 \right)}

\DeclareMathOperator{\pois}{Pois}
\DeclareMathOperator{\ber}{Ber}
\newcommand{\Zhat}{\widehat{Z}}

\renewcommand{\epsilon}{\varepsilon}

\def\ddefloop#1{\ifx\ddefloop#1\else\ddef{#1}\expandafter\ddefloop\fi}
\def\ddef#1{\expandafter\def\csname bb#1\endcsname{\ensuremath{\mathbb{#1}}}}
\ddefloop ABCDEFGHIJKLMNOPQRSTUVWXYZ\ddefloop

\def\ddefloop#1{\ifx\ddefloop#1\else\ddef{#1}\expandafter\ddefloop\fi}
\def\ddef#1{\expandafter\def\csname frak#1\endcsname{\ensuremath{\mathfrak{#1}}}}
\ddefloop ABCDEFGHIJKLMNOPQRSTUVWXYZ\ddefloop

\def\ddefloop#1{\ifx\ddefloop#1\else\ddef{#1}\expandafter\ddefloop\fi}
\def\ddef#1{\expandafter\def\csname fr#1\endcsname{\ensuremath{\mathfrak{#1}}}}
\ddefloop ABCDEFGHIJKLMNOPQRSTUVWXYZ\ddefloop

\def\ddefloop#1{\ifx\ddefloop#1\else\ddef{#1}\expandafter\ddefloop\fi}
\def\ddef#1{\expandafter\def\csname eul#1\endcsname{\ensuremath{\EuScript{#1}}}}
\ddefloop ABCDEFGHIJKLMNOPQRSTUVWXYZ\ddefloop

\def\ddefloop#1{\ifx\ddefloop#1\else\ddef{#1}\expandafter\ddefloop\fi}
\def\ddef#1{\expandafter\def\csname scr#1\endcsname{\ensuremath{\mathscr{#1}}}}
\ddefloop ABCDEFGHIJKLMNOPQRSTUVWXYZ\ddefloop

\def\ddefloop#1{\ifx\ddefloop#1\else\ddef{#1}\expandafter\ddefloop\fi}
\def\ddef#1{\expandafter\def\csname b#1\endcsname{\ensuremath{\mathbf{#1}}}}
\ddefloop ABCDEGHIJKLMNOPQRSTUVWXYZ\ddefloop

\def\ddefloop#1{\ifx\ddefloop#1\else\ddef{#1}\expandafter\ddefloop\fi}
\def\ddef#1{\expandafter\def\csname bhat#1\endcsname{\ensuremath{\hat{\mathbf{#1}}}}}
\ddefloop ABCDEFGHIJKLMNOPQRSTUVWXYZ\ddefloop

\def\ddefloop#1{\ifx\ddefloop#1\else\ddef{#1}\expandafter\ddefloop\fi}
\def\ddef#1{\expandafter\def\csname btil#1\endcsname{\ensuremath{\tilde{\mathbf{#1}}}}}
\ddefloop ABCDEFGHIJKLMNOPQRSTUVWXYZ\ddefloop

\def\ddefloop#1{\ifx\ddefloop#1\else\ddef{#1}\expandafter\ddefloop\fi}
\def\ddef#1{\expandafter\def\csname bst#1\endcsname{\ensuremath{\mathbf{#1}^\star}}}
\ddefloop ABCDEFGHIJKLMNOPQRSTUVWXYZ\ddefloop

\def\ddefloop#1{\ifx\ddefloop#1\else\ddef{#1}\expandafter\ddefloop\fi}
\def\ddef#1{\expandafter\def\csname bst#1\endcsname{\ensuremath{\mathbf{#1}^\star}}}
\ddefloop abcdeghijklmnopqrstuvwxyz\ddefloop

\def\ddefloop#1{\ifx\ddefloop#1\else\ddef{#1}\expandafter\ddefloop\fi}
\def\ddef#1{\expandafter\def\csname bhat#1\endcsname{\ensuremath{\hat{\mathbf{#1}}}}}
\ddefloop abcdefghijklmnopqrstuvwxyz\ddefloop

\def\ddefloop#1{\ifx\ddefloop#1\else\ddef{#1}\expandafter\ddefloop\fi}
\def\ddef#1{\expandafter\def\csname b#1\endcsname{\ensuremath{\mathbf{#1}}}}
\ddefloop abcdeghijklnopqrstuvwxyz\ddefloop

\def\ddefloop#1{\ifx\ddefloop#1\else\ddef{#1}\expandafter\ddefloop\fi}
\def\ddef#1{\expandafter\def\csname barb#1\endcsname{\ensuremath{\bar{\mathbf{#1}}}}}
\ddefloop abcdefghijklmnopqrstuvwxyz\ddefloop

\def\ddef#1{\expandafter\def\csname c#1\endcsname{\ensuremath{\mathcal{#1}}}}
\ddefloop ABCDEFGHIJKLMNOPQRSTUVWXYZ\ddefloop
\def\ddef#1{\expandafter\def\csname h#1\endcsname{\ensuremath{\widehat{#1}}}}
\ddefloop ABCDEFGHIJKLMNOPQRSTUVWXYZ\ddefloop
\def\ddef#1{\expandafter\def\csname hc#1\endcsname{\ensuremath{\widehat{\mathcal{#1}}}}}
\ddefloop ABCDEFGHIJKLMNOPQRSTUVWXYZ\ddefloop
\def\ddef#1{\expandafter\def\csname t#1\endcsname{\ensuremath{\widetilde{#1}}}}
\ddefloop ABCDEFGHIJKLMNOPQRSTUVWXYZ\ddefloop
\def\ddef#1{\expandafter\def\csname tc#1\endcsname{\ensuremath{\widetilde{\mathcal{#1}}}}}
\ddefloop ABCDEFGHIJKLMNOPQRSTUVWXYZ\ddefloop

\usepackage[suppress]{color-edits}
\addauthor{ab}{red}

\title{Partition Function Estimation under Bounded $f$-Divergence}
\usepackage{times}

\author[1,2]{Adam Block\thanks{adam.block@columbia.edu}}
\author[3]{Abhishek Shetty\thanks{shetty@mit.edu}}

\affil[1]{Department of Computer Science, Columbia University}
\affil[2]{Department of Electrical Engineering, Columbia University}
\affil[3]{MIT}
\date{}

\begin{document}

\maketitle

\begin{abstract}%

We study the statistical complexity of estimating partition functions given sample access to a proposal distribution and an unnormalized density ratio for a target distribution. While partition function estimation is a classical problem, existing guarantees typically rely on structural assumptions about the domain or model geometry. We instead provide a general, information-theoretic characterization that depends only on the relationship between the proposal and target distributions. Our analysis introduces the integrated coverage profile, a functional that quantifies how much target mass lies in regions where the density ratio is large. We show that integrated coverage tightly characterizes the sample complexity of multiplicative partition function estimation and provide matching lower bounds.  We further express these bounds in terms of $f$-divergences, yielding sharp phase transitions depending on the growth rate of f and recovering classical results as a special case while extending to heavy-tailed regimes. Matching lower bounds establish tightness in all regimes. As applications, we derive improved finite-sample guarantees for importance sampling and self-normalized importance sampling, and we show a strict separation between the complexity of approximate sampling and counting under the same divergence constraints. Our results unify and generalize prior analyses of importance sampling, rejection sampling, and heavy-tailed mean estimation, providing a minimal-assumption theory of partition function estimation.  Along the way we introduce new technical tools including new connections between coverage and $f$-divergences as well as a generalization of the classical Paley-Zygmund inequality.

\end{abstract}

\section{Introduction}\label{sec:intro}

There has been significant attention paid to the problem of partition function estimation in a number of specialized regimes, from classical models arising from statistical mechanics and combinatorics like the Ising model \citep{jerrum2003counting,vigoda2024sampling} to more general settings in scientific applications with physical structure imposed on $\lambda$ \citep{chipot2007free}.  
The goal, to estimate the normalizing constant of an unnormalized density has numerous applications throughout statistics, machine learning, and computer science, including Bayesian inference \citep{geweke1989bayesian,kass1995bayes}, graphical models \citep{wainwright2008graphical}, energy based models \citep{lecun2006tutorial}, statistical physics and chemistry \citep{chipot2007free,gelman1998simulating}, and  reinforcement learning for language model post-training \citep{chen2025coverage,brantley2025accelerating}, among many others.
Indeed, due to the significant interest, many prior works have been devoted to designing and analyzing algorithms to accomplish this important task.  
 While these works have provided numerous theoretical and practical insights in the problem domains studied, they often rely on structural assumptions on the structure of $\lambda$ (for example, through assumptions that arise in the context of learning on graphs) or the geometry of $\cX$ (for example, via smoothness or other regularity conditions on Euclidean space) that limit their applicability to more general settings.  
 Indeed, a surprising lacuna exists in the current literature: despite the fundamental nature of the partition function estimation problem, there is a dearth of general results that characterize the statistical complexity of partition function estimation in terms of natural and information theoretic properties of the underlying distributions $\mu$ and $\nu$.  
This is especially important given modern applications like language modeling, where the  domain is unstructured and $\lambda$ often corresponds to complex learned models and reward functions \citep{rafailov2023direct,xie2024exploratory}. 

In this work, we aim to fill this gap by asking:
    \emph{How many samples from a base distribution $\mu$ are required to estimate the partition function of a target distribution $\nu$ to a desired accuracy, as a function of natural information theoretic quantities between $\mu$ and $\nu$?}
We provide a complete answer to this question by providing tight bounds on the sample complexity $n$ required to achieve this estimate in terms of the coverage profile (\Cref{def:coverage}) between the target distribution $\nu$ and the proposal distribution $\mu$, as considered by \citep{chen2025coverage,chatterjee2018sample}.
The coverage profile quantitatively measures how the mass that $\nu$ places on regions where the density ratio $\nicefrac{d\nu}{d\mu}$ is large, and thus captures the tail behavior of the density ratio. 
In particular, we capture this decay in terms of a discrepancy measure between $\nu$ and $\mu$ we introduce and term \emph{integrated coverage}, $\icov[M]{\nu}{\mu}$ (\Cref{def:coverage}). We work in a setting where we have sample access to $\mu$ and can evaluate the \emph{unnormalized} density ratio $\lambda$, where $\lambda = Z \cdot \nicefrac{d\nu}{d\mu}$ for some unknown normalizing constant $Z = \int \lambda(x) d\mu(x)$.  We show that the integrated coverage precisely characterizes the sample complexity of partition function estimation as follows:
\begin{theorem}[Informal version of \Cref{thm:ub_coverage_coverage,thm:lb_coverage}]\label{thm:informal_coverage} 
    Let $\mu, \nu$ be two probability measures on $\cX$.  
    Then, let M be such that $ M^{-1} \cdot\icov[M]{\nu}{\mu} \leq \epsilon $, then $n = \Theta\left( M \cdot \epsilon^{-1} \right)$ 
    samples are necessary and sufficient to estimate $Z$ to multiplicative accuracy $(1 \pm \varepsilon)$. 
\end{theorem}
This result not only provides a sharp characterization of the sample complexity of partition function estimation in terms of natural information theoretic quantities, but also unifies, generalizes, and applies several prior results on importance sampling and mean estimation \citep{chatterjee2018sample,devroye2016sub}.

The sharpness of our bound can be further interpreted in terms of more familiar notions of discrepancy between $\nu$ and $\mu$, \emph{$f$-divergences} \citep{csiszar2011information,polyanskiywu_itbook}.
Recall that $f$-divergence (\Cref{def:fdiv}) is defined as the expectation of a convex function $f$ of the density ratio, and generalizes notions such as total variation, KL-divergence, and Renyi divergences.
Indeed, an informal version of our main result stated in terms of $f$-divergences is as follows:
\begin{theorem}[Informal statement of \Cref{thm:ub_fdiv,thm:lb}]\label{thm:informal}
    Given two probability measures $\mu, \nu$ on $\cX$ and an $f$-divergence $\divf{\nu}{\mu}$ between them, there is a function $\gamma_f$ depending on $f$ such that
    \begin{align}
        n = \Theta\left( \left[ \gamma_f\left( \Theta(1) \cdot \epsilon^{-1} \cdot \divf{\nu}{\mu} \right) \right] \vee  \left[ \chisq{\nu}{\mu} \cdot \epsilon^{-2} \right] \right)
    \end{align}
    samples are necessary and sufficient to estimate $Z$ to multiplicative accuracy $(1 \pm \varepsilon)$ with constant probability, where $\chisq{\cdot}{\cdot}$ is the $\chi^2$-divergence.
\end{theorem}

As an application of these results, we provide a sharper finite-sample analysis of importance sampling and self-normalized importance sampling (SNIS) estimators in terms of the target function and the distributions $\mu$ and $\nu$ (\Cref{thm:is,thm:snis}) and provides a unified perspective on prior results \citep{owen2013monte,chatterjee2018sample}.  
In particular, our results could help inform the design of proposal distributions for importance sampling that minimize the required sample complexity for a given set of target functions that is more flexible than the classical variance-minimizing proposal distribution \citep{owen2013monte,llorente2025optimalityimportancesamplinggentle}.

We also make an interesting connection to the sample complexity of sampling from $\nu$ given samples from $\mu$ given access to unnormalized density ratios $\lambda$. 
Generalizing results by \cite{block2023sample,flamich2024some}, we show that we provide a tight characterization of the sample complexity of sampling in terms of coverage and $f$-divergence as follows:

\begin{theorem}[Informal version of \Cref{prop:sampling}]
    \label{prop:sampling_inf}
    Let $\mu, \nu$ be probability measures on $\cX$. Let $M$ be such that $\covm[M]{\nu}{\mu} \epsilon$. Then, $ n = \tilde{\Theta}\left( M \cdot \log\left( \nicefrac 1\epsilon \right) \right)$ are necessary and sufficient to produce $\epsilon$-approximate samples from $\nu$ in total variation distance given samples from $\mu$ and access to an unnormalized density ratio. 
    In particular, this holds for $n \gtrsim \log\left( \nicefrac 1\epsilon \right) \cdot \gamma_f\left(\Theta(1) \cdot \nicefrac{ \divf{\nu}{\mu}}{\epsilon} \right)$. 
\end{theorem}

    This result shows that sampling is strictly easier than estimation, in conceptual constrast to settings (such as self-reducibility) where sampling is often approximately has the same complexity as estimation \citep{jerrum2003counting,vigoda2024sampling}. 
    We provide a detailed comparison of the sample complexity of sampling and estimation in \Cref{ssec:est_sampling}.

We begin in the next section by formally setting up the problem of interest and introducing and defining the key quantitites used throughout the paper.  We then proceed to state our main results in \Cref{sec:main_results}, including both upper and lower bounds on the sample complexity of partition function estimation, as well as bounds on approximate sampling.  We proceed in \Cref{sec:snis} to discuss applications of our results to importance sampling and self-normalized importance sampling.  We then provide a high-level overview of the proof techniques used to establish our main upper bounds in \Cref{sec:proof_ideas}, deferring full proofs of all results to the appendix.  We conclude in \Cref{sec:discussion} with a discussion of open questions and future directions.  Our proofs introduce several novel technical tools of independent interest, including a connection between $f$-divergences and integrated coverage, a variance bound for truncated density ratios (\Cref{lem:variance_ub}), and a strong generalization of Paley-Zygmund (\Cref{lem:gen_paley_zygmund}).

\section{Problem Setup and Preliminaries}\label{sec:prelims}

 We consider a measurable space $\cX$ and base probability measure $\mu$ over $\cX$.  
 We will also consider a target distribution $\nu$ over $\cX$ from which we would either like to produce a sample or to estimate its normalizing constant given access to samples from $\mu$ and an unnormalized density ratio $\lambda: \cX \to \rr_{\geq 0}$ defined using the Radon-Nikodym derivative of $\nu$ with respect to $\mu$ as $\lambda(x) = Z \cdot \nicefrac{d\nu}{d\mu}(x)$ for some unknown normalizing constant $Z$. 
We are primarily interested in understanding the difficulty of obtaining an estimate $\Zhat$ of the normalizing constant $Z$; more precisely, we ask the following question:
    \emph{Given access to i.i.d. samples $X_1, \ldots, X_n \sim \mu$ and the ability to evaluate an unnormalized density ratio $\lambda$ on each sample, how large must $n$ as a function of $\mu $, $\nu$, $\epsilon$ and $\delta$  be to ensure  that there exists an estimator $\Zhat = \Zhat(X_1, \ldots, X_n, \lambda(X_1), \ldots, \lambda(X_n))$ satisfies $(1 - \epsilon) Z \leq \Zhat \leq (1 + \epsilon) Z$ with probability at least $1-\delta$?}

The rate at which such an estimate can be obtained depends on the similarity between the two distributions $\mu$ and $\nu$.
For example, if $\mu$ and $\nu$ had disjoint supports, then no finite number of samples from $\mu$ would suffice to estimate $Z$.  
The main focus of this work is provide a sharp characterization of the sample complexity in terms of quantitative measures of similarity between $\mu$ and $\nu$.

The key notion measuring the divergence between $\mu$ and $\nu$ is that of \emph{coverage}, considered by \cite{chatterjee2018sample,chen2025coverage}.  
Conceptually, coverage measures the mass that $\nu$ places on regions where the density ratio $\frac{d\nu}{d\mu}$ is large and thus good coverage implies that $\mu$ places sufficient mass in high-density regions of $\nu$.  Formally, coverage is defined as follows.

\begin{definition}[Coverage and Integrated Coverage] 
    \label{def:coverage}
    For two probability measures $\mu, \nu$ on $\cX$ and $M > 0$, the \emph{coverage} function at $M$ is defined to be 
    \begin{align}
        \covm{\nu}{\mu} = \nu\left( \left\{ x \in \cX : \nicefrac{d\nu}{d\mu}(x) \geq M \right\} \right),
    \end{align}
    where $\nicefrac{d\nu}{d\mu}$ is the Radon-Nikodym derivative of $\nu$ with respect to $\mu$.
    Further, we introduce a notion that we term the integrated coverage profile as 
    \begin{align}
        \icov{\nu}{\mu}  = \int_{0}^{M} \covm[t]{\nu}{\mu} dt.%
    \end{align}
\end{definition}
We first note several basic properties of coverage and integrated coverage.  First,
it is immediate that $\covm{\nu}{\mu} \leq 1$ for all $M > 0$, and thus $\icov{\nu}{\mu} \leq M$; moreover, $\covm{\nu}{\mu}$ is clearly non-increasing in $M$.  Furthermore, an easy calculation implies that $M \mapsto \nicefrac{\icov[M]{\nu}{\mu}}M$ is also non-increasing in $M$ and whenever $\nu \ll \mu$, this quantity tends to zero as $M \to \infty$.  We will see that this latter map is fundamental to characterizing the sample complexity of partition function estimation.

The notion of coverage has previously been used to analyze the sample complexity of importance sampling \citep{chatterjee2018sample}, the sample complexity of rejection sampling \cite{block2023sample} and, interestingly, for understanding the efficacy of pretrained language models \cite{chen2025coverage,huang2025self,huang2025best}. 
The notion of integrated coverage, introduced here, is a more refined measure of the relationship between $\mu$ and $\nu$ that allows for sharper characterizations of sample complexity.

A more standard notion of discrepancy between two measures is the $f$-divergence~\citep{csiszar2011information,polyanskiywu_itbook}, which is defined as follows.
\begin{definition}\label{def:fdiv}
    Let $f: \rr_{\geq 0} \to \rr_{\geq0}$ be a convex function satisfying $f(1) = f'(1) = 0$. For two probability measures $\mu, \nu$ on $\cX$, the $f$-divergence between $\nu$ and $\mu$ is defined to be
    \begin{align}\label{eq:fdiv}
        \divf{\nu}{\mu} = \ee_{X \sim \mu}\left[ f\left( \frac{d\nu}{d\mu}(X) \right) \right] + \nu\left( \frac{d \nu}{d\mu} = \infty \right) \cdot f'(\infty).
    \end{align}
\end{definition}
The $f$-divergences generalize many well-known notions of distance between probability measures, including total variation distance (where $f(t) = \nicefrac{1}{2} \abs{t - 1}$), the Kullback-Leibler (KL) divergence (where $f(t) = t \log t$), and (a monotone transformation of) the $\alpha$-Renyi divergences (where $f(t) = t^{\alpha} - \alpha t$ for $\alpha > 0$, $\alpha \neq 1$).  Of special note is the $\alpha$-Renyi divergence with $\alpha = 2$, which corresponds to the $\chi^2$ divergence, defined as $\chi^2(\nu \| \mu) = \ee_{X \sim \mu}\left[ \left( \nicefrac{d\nu}{d\mu}(X) - 1 \right)^2 \right]$.

Conceptually, a bounded $f$-divergence controls the tail behavior of the density ratio $\frac{d\nu}{d\mu}$ with faster growing $f$ leading to stronger control.  As we show in \Cref{sec:lem_cov_fdiv_proof}, this intuition is made precise in several ways, most critically in the observation that $\covm{\nu}{\mu} \leq \nicefrac{M \cdot \divf{\nu}{\mu}}{M}$ for any $f$-divergence; thus, bounded $f$-divergence implies coverage decays quickly.  As we shall see, our sample complexity results involving the $f$-divergences  will depend on the growth rate of $f$ through that which we term the $\gamma_f$ function, defined to be the inverse of the map $t \mapsto \nicefrac{f(t)}{t}$ on $[1, \infty)$, i.e.,  
\begin{align}\label{eq:gammaf}
    \gamma_f(M) = \inf \left\{ t \geq 1 : \nicefrac{f(t)}{t} \geq M \right\}.
\end{align} 
The growth rate of $f$ determines the behavior of $\gamma_f$; for example, 
when $f$ is \emph{superlinear}, i.e., $\lim_{t \to \infty} \nicefrac{f(t)}{t} = \infty$, the function $\gamma_f$ is well-defined on all of $\rr_{\geq 0}$; examples of superlinear $f$ include those corresponding to KL divergence (where $\gamma_f(M) \asymp \exp(M)$) and $\alpha$-Renyi divergences with $\alpha > 1$ (where $\gamma_f(M) \asymp M^{\frac{1}{\alpha - 1}}$).   We will also have occasion to distinguish between \emph{superquadratic} $f$, where $\lim_{t \to \infty} \nicefrac{f(t)}{t^2} = \infty$ (e.g., $\alpha$-Renyi divergences with $\alpha > 2$) and those that are not (e.g., KL divergence and $\alpha$-Renyi divergences with $1 < \alpha \leq 2$ such as $\chi^2$).  In the case that $f$ is \emph{linear}, we let $\gamma_f(M) = \infty$ for $M > \sup_{t} \nicefrac{f(t)}{t}$; note that by the assumption of convexity, $f$ cannot be sublinear.
These distinctions will control the precise rates in our sample complexity results.

\section{Main Results}\label{sec:main_results}

We now present our main results regarding the sample complexity of estimating the partition function given access to samples from the base distribution and an unnormalized density ratio.

\subsection{Upper Bounds for Estimation}

We begin by stating our main upper bound on the sample complexity of estimating the partition function in terms of the integrated coverage profile defined in \Cref{def:coverage}.

\begin{theorem}\label{thm:ub_coverage_coverage}
    Let $\mu, \nu$ be probability measures on $\cX$.  Suppose that we have access to i.i.d. samples $X_1, \ldots, X_n \sim \mu$ and an unnormalized density ratio $\lambda$.  
    Let $M_{\epsilon}$ be such that $ M^{-1} \cdot \icov{\nu}{\mu} \leq \epsilon/4 $.  
    Then there is an estimator $\Zhat$ such that as long as
    \begin{align}\label{eq:sample_complexity_ub_cov}
        n \gtrsim  \nicefrac{ M_{\epsilon} \cdot \log\left( \nicefrac 1\delta \right)}{\epsilon} ,
    \end{align}
    then with probability at least $1 - \delta$ it holds that $(1 - \epsilon) \cdot Z \leq \Zhat \leq (1 + \epsilon) Z$.
\end{theorem}
As we discussed above, as long as $\nu \ll \mu$, such an $M_\epsilon$ always exists for any positive $\epsilon$; thus the above theorem always provides a finite sample complexity bound for estimating the partition function to $(1 \pm \epsilon)$ accuracy in the absolutely continuous setting.

In order to appreciate \eqref{eq:sample_complexity_ub_cov}, consider the simple case when $ \chi^2 ( \nu || \mu)  $ is bounded. 
In this case, a simple tail bound implies $ \icov{\nu}{\mu} \leq \chi^2 ( \nu || \mu) $ for all $M$. 
Therefore, we can set $ M_{\epsilon} \asymp \chi^2 ( \nu || \mu)  /  \epsilon $ in Theorem~\ref{thm:ub_coverage_coverage}, and obtain that $ n \gtrsim \chi^2 ( \nu || \mu)  / \epsilon^2 $ samples suffice to estimate the partition function to $(1 \pm \epsilon)$ accuracy with high probability.
This recovers standard results bounding the sample complexity of importance sampling by computing the variance \citep{owen2013monte}. 
It is natural to ask, then, how slowly can $M_\epsilon$ grow as a function of $\epsilon$?  Perhaps one might hope that with improved moment bounds, one could achieve a faster estimation rate via $M_\epsilon$, but this is not the case.  Indeed, this can be seen intuitively by observing that, because $ \covm[0]{\nu}{\mu} = 1 $ for any two distributions $ \nu, \mu $, we must have $ M_{\epsilon} \gtrsim \nicefrac 1\epsilon$ as $\epsilon \downarrow 0$ and \Cref{thm:ub_coverage_coverage} can never achieve a sample complexity scaling better than $ \Omega(\epsilon^{-2}) $, as one would expect for estimating a mean to $ (1 \pm \epsilon) $ accuracy due to the central limit theorem.   
Thus, we should instead view the above theorem as a generalization of the bound to more general heavy-tailed settings where the $ \chi^2 $ divergence may be infinite.

We may thus ask whether we can provide more explicit bounds on the sample complexity in terms of the more familiar notion of $f$-divergences, which amount to control on the tails of the density ratio.  This is the content of our next theorem.

\begin{theorem}\label{thm:ub_fdiv}
    Let $\mu, \nu$ be probability measures on $\cX$ and let $f$ be a convex function as in \Cref{def:fdiv} such that there exists\footnote{The actual constant $c \geq 1$ in \Cref{thm:ub_fdiv} is a technical artifact and should generally thought to be equal to 1.  Indeed, this condition is simply ensuring that $f(t)$ is eventually subquadratic, as discussed below.} $c \geq 1$ such that for $t \geq c$ the function $t \mapsto f(t)/t^2$ is non-increasing.  Suppose that we have access to i.i.d. samples $X_1, \ldots, X_n \sim \mu$ and an unnormalized density ratio $\lambda$.  Then there is an estimator $\Zhat$ such that as long as
    \begin{align}\label{eq:sample_complexity_ub}
        n \gtrsim  \frac{ \gamma_f\left( \nicefrac{6 \cdot \divf{\nu}{\mu}}{\epsilon} \right) \cdot \log\left( \nicefrac 1\delta \right)}{\epsilon} \vee \frac{c^2}{\epsilon^2},
    \end{align}
    then with probability at least $1 - \delta$ it holds that $(1 - \epsilon) \cdot Z \leq \Zhat \leq (1 + \epsilon) Z$.
\end{theorem}

\begin{wrapfigure}{r}{0.45\textwidth}
    \centering
    \vspace{-10pt}
    \includegraphics[width=0.43\textwidth]{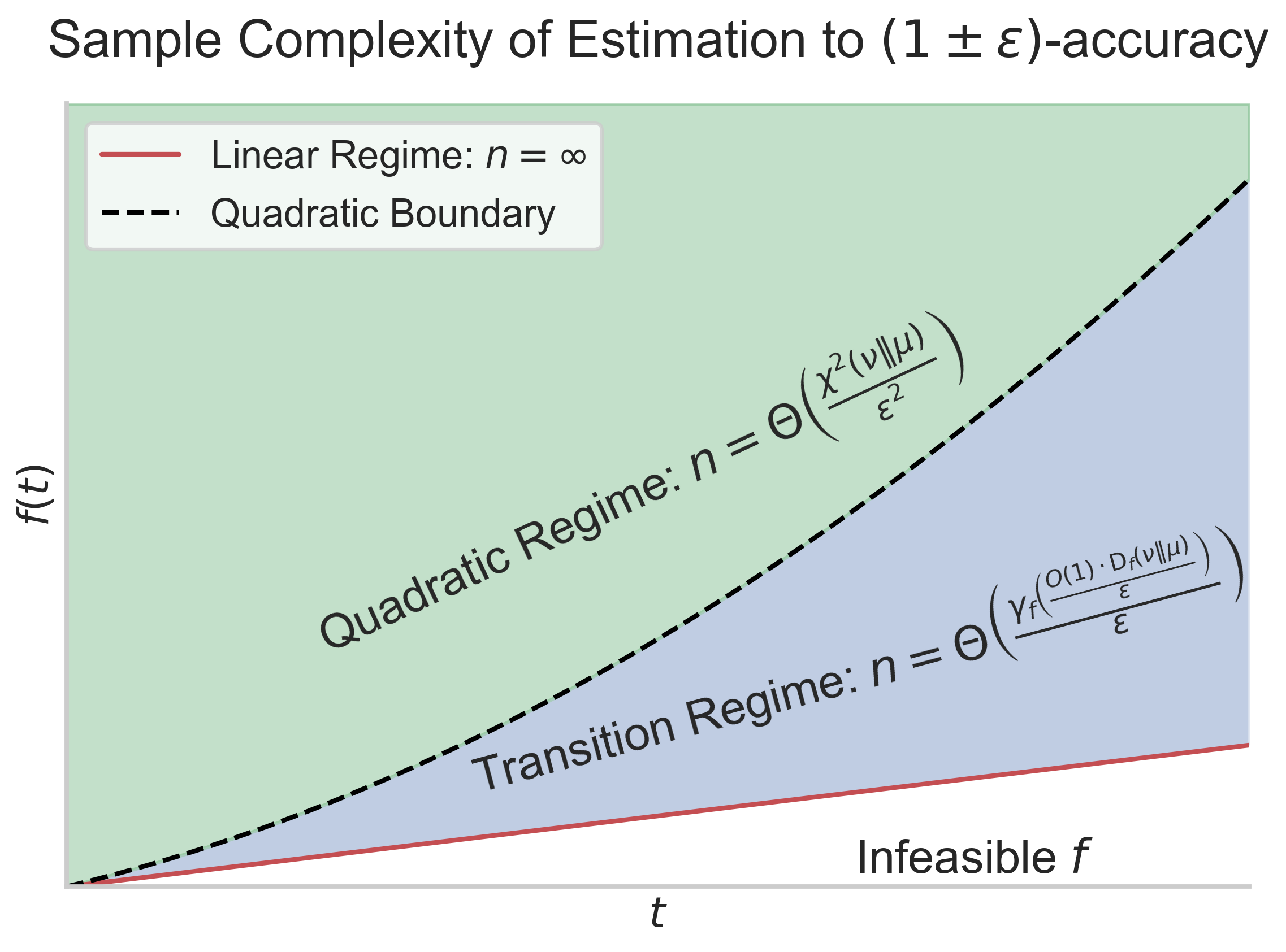}
    \caption{Sample complexity ($n$) regimes for estimating $Z$ to $(1 \pm \epsilon)$ accuracy based on the growth rate of the $f$ defining $\divf{\nu}{\mu}$.}
    \vspace{-20pt}
    \label{fig:sample_complexity_regions}
\end{wrapfigure}

The $f$-divergence perspective provides a complementary view to the coverage-based bound of \Cref{thm:ub_coverage_coverage}. 
We can separate the sample complexity of estimating the partition function into three regimes (summarized in \Cref{fig:sample_complexity_regions}) depending on the growth rate of $f$: \emph{linear}, \emph{superlinear} but \emph{subquadratic}, and \emph{superquadratic}.  
Indeed, in the case where $f$ is linear, i.e., $\lim_{t \to \infty} \nicefrac{f(t)}{t} < \infty$, the function $\gamma_f$ is only defined on a bounded domain and thus for sufficiently small $\epsilon$, the theorem is vaccuous; for example, this is the case for total variation or Hellinger distances.  In the superlinear but subquadratic case, where $\lim_{t \to \infty} \nicefrac{f(t)}{t^2} = C < \infty$, we see that $\gamma_f$ is defined on all of $\rr_{\geq 0}$ and grows at least as quickly as $\epsilon^{-1}$, implying that the first term of \eqref{eq:sample_complexity_ub} dominates.  This is the case for KL divergence where $n \gtrsim \exp(c \cdot \nicefrac{\dkl{\nu}{\mu}}{\epsilon}) / \epsilon$ samples suffice as well as for $\alpha$-Renyi divergences with $1 < \alpha \leq 2$, where $n \gtrsim \epsilon^{-\alpha} \cdot \divf{\nu}{\mu}^{\nicefrac{1}{\alpha - 1}}$ samples suffice.  Finally, in the superquadratic case where $\lim_{t \to \infty} \nicefrac{f(t)}{t^2} = \infty$, we see that $\gamma_f$ grows too slowly to dominate the sample complexity for small $\epsilon$ and thus the second term of \eqref{eq:sample_complexity_ub} kicks in; this is the case for $\alpha$-Renyi divergences with $\alpha > 2$, where $n \gtrsim 1 / \epsilon^2$ samples suffice.

Note that our tightest upper bounds can never improve on the $\epsilon^{-2}$ scaling for small $\epsilon$, as discussed above.  We thus conclude the discussion of our upper bounds with an \emph{asymmetric} approach that is capable of improving the sample complexity scaling in $\epsilon$ for the lower tail at the cost of a significantly worse upper tail bound.

\begin{theorem}\label{thm:ub_lowertail}
    Let $\mu, \nu$ be probability measures on $\cX$ and suppose $X_1, \dots, X_n \sim \mu$ are independent with $\lambda$ an unnormalized density ratio. Then there is an estimator $\Zhat$ such that for any $0 < \epsilon < 1$, if
    \begin{align}\label{eq:sample_complexity_ub_lowertail}
        n \gtrsim \log\left( \nicefrac{1}{\delta} \right) \cdot \nicefrac{M}{\epsilon}  \quad \text{where } \Cov_M\left( \nu \| \mu \right) \lesssim \epsilon,
    \end{align}
    it holds that with probability at least $1 - \delta$ that $(1 - \epsilon) \cdot Z \leq \Zhat \leq M \cdot Z$.  In particular, if $\divf{\nu}{\mu} < \infty$, then it suffices to take $M = \gamma_f\left( \nicefrac{6 \cdot \divf{\nu}{\mu}}{\epsilon} \right)$.
\end{theorem}
Note that the $M$ in \eqref{eq:sample_complexity_ub_lowertail} is no larger than $M_\epsilon$ in \Cref{thm:ub_coverage_coverage}, because $\icov{\nu}{\mu} \leq M \cdot \covm{\nu}{\mu}$ for all $M$ by the fact that coverage is decreasing.  Thus, in the setting where $f$ is superquadratic, we can achieve a sample complexity of estimating $Z$ from below that is significantly smaller than the $\epsilon^{-2}$ scaling of \Cref{thm:ub_coverage_coverage,thm:ub_fdiv}.

\subsection{Lower Bounds}

In order to complement our upper bounds, we provide several lower bounds on the sample complexity of estimating the partition function, both in terms of integrated coverage and in terms of the $f$-divergences that demonstrate the tightness of our bounds. The lower bound constructions, deferred to \Cref{app:lower_bound_proof}, are inspired in part by \citet{block2023sample} and in particular the result of \citet{harremoes2011pairs} that suggests that pairs of Bernoulli random variables extremize the joint range of $f$-divergences.  We begin with the integrated coverage lower bound.
\begin{theorem}\label{thm:lb_coverage}
    For any non-atomic $\mu$ and any $\epsilon \leq \nicefrac 12$, there exists a family of distributions $\nu$ such that any estimator $\Zhat$ that achieves $(1 \pm \epsilon)$ multiplicative accuracy with probability at least $\nicefrac 23$ requires at least $ \epsilon^{-1} \cdot M_\epsilon$ samples, where $M_\epsilon$ satisfies
    \begin{align}
        \icov[M_{\epsilon}]{\nu}{\mu} \leq M_\epsilon \cdot \epsilon.
    \end{align}
\end{theorem}
In particular, \Cref{thm:lb_coverage} shows that the sample complexity bound of \Cref{thm:ub_coverage_coverage} is tight in its dependence on integrated coverage and $\epsilon$.  Unfortunately, \Cref{thm:lb_coverage} does not directly translate to lower bounds in terms of $f$-divergences; while it is indeed the case that integrated coverage can be tightly controlled via $f$-divergence, it is not clear that this tight control is realized by the lower bound construction of \Cref{thm:lb_coverage}.  We thus complement \Cref{thm:ub_fdiv} with separate lower bounds for each of the three regimes discussed above, beginning with the linear case.
\begin{proposition}\label{prop:lb_linear}
    Let $f$ be a convex function as in \Cref{def:fdiv} such that $t \mapsto \nicefrac{f(t)}t$ is bounded.  For any $0 < C < f'(\infty)$, there exist measures $\mu, \nu$ such that $\divf{\nu}{\mu} \leq C$ such that no estimator given access to finitely many samples from $\mu$ is capable of producing a $1 \pm \epsilon$ approximation of $Z$ with probability at least $\nicefrac 34$ for any $\epsilon \leq \nicefrac{f'(\infty)}{(f'(\infty) - C)}$.  Moreover, if $C \geq f'(\infty)$, then there exist measures $\mu, \nu$ such that $\divf{\nu}{\mu} \leq C$ and no such estimator can produce a $1 \pm \epsilon$ approximation of $Z$ with probability at least $\nicefrac 34$ for any $\epsilon > 0$.
\end{proposition}
\Cref{prop:lb_linear} shows that when $f$ is linear, no finite number of samples suffices to estimate the partition function to within any nontrivial multiplicative accuracy.  This fact is unsurprising given that $f$-divergences with linear $f$ (such as total variation or Hellinger) do not control the tails of the density ratio $\frac{d\nu}{d\mu}$ in any meaningful way and indeed allow for the possibility that $\nu$ is singular with respect to $\mu$, where clearly estimation of $Z$ is impossible.  We now proceed to the second case, where $f$ is superlinear but subquadratic.
\begin{theorem}\label{thm:lb}
    For any $C > 2$ and $\epsilon \in (0, \frac{1}{4})$, there exists a probability measure $\mu$ on $\cX$ and class of probability measures $\cV$ on $\cX$ such that for all $\nu \in \cV$, it holds that $\divf{\nu}{\mu} \leq C + f(1 - \epsilon)$ and any estimator $\Zhat$ that satisfies $(1 - \epsilon) Z \leq \Zhat \leq (1 + \epsilon) Z$ with probability at least $\frac{2}{3}$ for all $\nu \in \cV$ must use at least $ n \gtrsim \nicefrac{\gamma_f\left( \frac{C}{2 \cdot \epsilon} \right)}{\epsilon}$
    samples.
\end{theorem}
While this lower bound holds for arbitrary $f$-divergences, it is only tight in the superlinear/subquadratic regime, where $\gamma_f(\nicefrac 1\epsilon) \cdot \epsilon^{-1} \gg \epsilon^{-2}$; this applies to KL-divergence and Renyi divergences for $1 < \alpha \leq 2$.  For the superquadratic case, we have the following lower bound. 
\begin{proposition}\label{prop:superquadratic_lb}
    Fix $0 < \epsilon < \nicefrac 18$.  Then there exist a family of probability measures $\cU$ on $\cX$ and  probability measure $\nu$ on $\cX$ such that for all $\mu \in \cU$, it holds that $\norm{\nicefrac{d\nu}{d\mu}}_{\infty} \leq 4$ and any estimator $\Zhat$ that for all $\nu \in \cV$  with probability at least $\nicefrac{2}{3}$ is a $1 \pm \epsilon$ approximation of $Z$  must use at least $n \gtrsim \frac{1}{\epsilon^2}$ samples.
\end{proposition}
Taken together, the lower bounds presented in this section demonstrate that each of our upper bounds are tight in their respective regimes both in terms of integrated coverage and $f$-divergences.

\subsection{Comparison to Sampling}\label{ssec:est_sampling}

While our primary focus is on partition function estimation, it is instructive to compare our results to those for the related problem of \emph{sampling} from the target distribution $\nu$ given access to samples from the base distribution $\mu$ and an unnormalized density ratio $\lambda$.  In \citet{block2023sample}, the authors examined the question of \emph{approximate rejection sampling}, where the learner is given access to $n$ i.i.d. samples from $\mu$ and the \emph{normalized} density ratio $\frac{d\nu}{d\mu}(\cdot)$ and must produce a sample from a distribution $\nu_n$ that is close to $\nu$ in total variation distance under the assumption of bounded $f$-divergence between $\nu$ and $\mu$.  We generalize their result, as well as that of \citet{flamich2024some} to the setting of \emph{unnormalized} density ratios and arbitrary $f$-divergences.

\begin{proposition}\label{prop:sampling}
    Let $\mu, \nu$ be probability measures on $\cX$ and let $f$ be a convex function as in \Cref{def:fdiv}.  Suppose that we have access to i.i.d. samples $X_1, \ldots, X_n \sim \mu$ and an unnormalized density ratio $\lambda$.  Then there is an algorithm that produces a sample $X_{\jhat}$ from a distribution $\nu_n$ such that $\tv\left( \nu_n, \nu \right) \leq \epsilon$ as long as
    \begin{align}\label{eq:sampling_sample_complexity_coverage}
        n \gtrsim M \cdot \log\left( \nicefrac 1\epsilon \right) \quad \text{where } \Cov_M\left( \nu \| \mu \right) \lesssim \epsilon.
    \end{align}
    In particular, this holds for
    \begin{align}\label{eq:sampling_sample_complexity_fdiv}
        n \gtrsim \log\left( \nicefrac 1\epsilon \right) \cdot \gamma_f\left(\Theta(1) \cdot \nicefrac{ \divf{\nu}{\mu}}{\epsilon} \right).
    \end{align}
\end{proposition}
We observe that the lower bounds in \citet{block2023sample} for approximate rejection sampling already demonstrate that \eqref{eq:sampling_sample_complexity_fdiv} is tight up to logarithmic factors as the regime in that work is strictly stronger than ours (i.e., access to the normalized density ratio rather than the unnormalized one).  We defer a proof of \Cref{prop:sampling} to \Cref{sec:sampling_proof}, where we generalize analysis of the $A^*$-sampling algorithm in \citet{li2018strong,flamich2024some}.

It is instructive to note that the sample complexity of sampling is strictly smaller than that of partition function estimation. 
To see this note that $ \covm[M]{\nu}{\mu}  \leq M^{-1} \cdot \icov[M]{\nu}{\mu} $ for all $M$, since $\covm[M]{\nu}{\mu}$ is a decreasing function of $M$.    Thus it follows immediately that we have at least an $\epsilon^{-1}$ factor improvement in sample complexity when comparing \eqref{eq:sampling_sample_complexity_coverage} to \eqref{eq:sample_complexity_ub} in \Cref{thm:ub_coverage_coverage}.  The improvement can be even more stark, however, depending on the behavior of the coverage profile.  Indeed, we saw that $M^{-1} \cdot \icov[M]{\nu}{\mu}$ can never shrink more quickly than $M^{-1}$ as $M$ increases; thus, in the superquadratic regime, where $\divf{\nu}{\mu}$ is bounded for quickly-growing $f$, we can have almost a quadratic separation between the sample complexity of sampling and that of partition function estimation.  Indeed, in the extreme case where the density ratio is uniformly bounded, i.e., $\norm{\nicefrac{d\nu}{d\mu}}_{\infty} \leq M$, sampling requires only $n \gtrsim  \log\left( \nicefrac 1\epsilon \right)$ samples to produce an $\epsilon$-approximate sample from $\nu$ in total variation distance, while partition function estimation requires $n \gtrsim \epsilon^{-2}$ samples. One way to interpret this is that estimation depends on the entire coverage profile, while sampling only depends on the coverage at a single value of $M$.

In this way we see that, in contradistinction to the well-known class of `self-reducible' problems, the problem of `counting' (i.e., partition function estimation) is strictly harder than that of `sampling' under general $f$-divergence or coverage constraints.

\section{Application: Improved Finite Sample Bounds for Importance Sampling}\label{sec:snis}

Importance sampling is a fundamental technique in statistics and machine learning which serves as a primitive in many domains, including  causal inference, reinforcement learning \citep{precup2000eligibility, thomas2015high}, variational inference \citep{burda2015importance}, and probabilistic programming \citep{wingate2011lightweight}.
It serves as a tool for estimating expectations of a function $g$ (we focus on the bounded, positive case)  under a target distribution $\nu$   using samples from a proposal distribution $\mu$.
Given i.i.d.\ samples $X_1, \ldots, X_n \sim \mu$, the importance sampling estimator for $ \nu_g =  \ee_{\nu}[g(X)]$ is
\begin{align}
    \hat{g}_{\mathrm{IS}} = \frac{1}{n} \sum_{i=1}^n \frac{d \nu}{d \mu}(X_i) g(X_i).
\end{align}
This estimator is unbiased, i.e., $\ee\left[ \hat{g}_{\mathrm{IS}} \right] = \nu_g$ and standard results on importance sampling estimators bound the variance of $\hat{g}_{\mathrm{IS}}$ which can be bounded rely on the $\chi^2$-divergence between the target and proposal distributions, since the variance is bounded by $\chi^2(\nu \| \mu)$.

Towards obtaining sharper bounds, we consider the target distribution weighted by the function $g$, i.e., the measure $\nu \cdot g$ defined with density given by
$ \nicefrac{g(x)}{ \nu_g} \cdot \nicefrac{d \nu}{d \mu}(x)$ with respect to $\mu$. 
This allows us to define the coverage between the weighted target distribution $\nu \cdot g$ and the proposal distribution $\mu$, which in this context maybe expressed as 
\begin{align}
    \covm{ \nu \cdot g }{\mu} = \frac{1}{ \nu_g } \int g(x) \frac{d \nu}{d \mu}(x) \mathbb{I} \left[  g(x) \frac{d \nu}{d \mu}(x) \geq M \cdot \nu_g \right]  d\mu(x) 
\end{align}
Further this allows us to define the integrated coverage and the $f$-divergence $\divf{\nu \cdot g}{\mu}$  between the weighted target distribution $\nu \cdot g$ and the proposal distribution $\mu$ in the natural way. 

Applying \Cref{thm:ub_coverage_coverage}, we obtain finite-sample bounds for importance sampling estimators in terms of general $f$-divergences. 

\begin{theorem}[Importance Sampling] \label{thm:is}
    Let $\mu, \nu$ be measures and $g$ be a bounded, positive function.
    Then the importance sampling estimator $\hat{g}_{\mathrm{IS}}$ satisfies $\abs{\hat{g}_{\mathrm{IS}} - \nu_g } \leq \epsilon \cdot \nu_g$ with probability at least $1 - \delta$, as long as $n \gtrsim \nicefrac{  M_{\epsilon, \delta} }{\epsilon}$.
    where $M_{\epsilon, \delta}$ satisfies $\icov[M_{\epsilon, \delta}]{ g \nu }{\mu} \leq \nicefrac{\epsilon \delta }{6}$.
\end{theorem}

This generalizes the standard variance bounds for importance sampling, which correspond to the case of the $\chi^2$-divergence.
The dependence on probability parameter $\delta$ can be improved to $\log (1/\delta ) $ using the median of means technique as in the proof of \Cref{thm:ub_coverage_coverage} but state it in this form to maintain the form of the estimator that is commonly used in practice.

More importantly, our results bound the error as jointly as function of $g , \nu$ and $\mu$, rather than separately e.g. in terms of $D_f(\nu \| \mu)$ and $\|g\|$ \cite{chatterjee2018sample}. 
This is particularly useful in settings when we get to design the proposal distribution $\mu$ based on the target functions $g$ and distribution $\nu$, bound above can be used to optimize the proposal distribution $\mu$ to minimize the required sample complexity for a given set of target functions $g$. 
In particular, the optimal proposal distribution is a well studied problem in importance sampling (See \citep{owen2013monte, llorente2025optimalityimportancesamplinggentle} for a survey-level treatment) and is typically chosen to minimize the variance of the importance sampling estimator. 
Our results suggest that a more refined objective is to minimize the integrated coverage between the weighted target distribution.

Our techniques also extend to self-normalized importance sampling (SNIS) \cite{owen2013monte}, which estimates expectations $\ee_{\nu}[g(X)]$ when we can only sample from a proposal distribution $\mu$ and evaluate the unnormalized density ratio $\lambda(x) = \frac{d\nu}{d\mu}(x) \cdot Z$ for some unknown normalizing constant $Z = \int \lambda(x) d\mu(x)$.
Given i.i.d.\ samples $X_1, \ldots, X_n \sim \mu$, the SNIS estimator is $\hat{g}_{\mathrm{SNIS}} =  \nicefrac{ \sum_{i=1}^n \lambda(X_i) g(X_i)  }{ \sum_{i=1}^n \lambda(X_i) }$.
Under the assumption that $\mu$ is absolutely continuous with respect to $\nu$, the SNIS estimator is asymptotically unbiased \cite[Theorem 9.2]{owen2013monte}.
The asymptotic variance is $\mathbb{E}_{X \sim \mu}\left[ \left( g(X) - \mathbb{E}_{X \sim \nu}[g(X)] \right)^2 \cdot \left(\frac{d\nu}{d\mu}(X) \right)^2 \right]$ 
which for bounded functions $g$ is at most $\|g\|_{\infty}^2 \cdot \chi^2(\nu \| \mu)$.
These asymptotic bounds yield finite-sample sample complexity of $n \geq \frac{ 4 \|g\|_{\infty}^2 \chi^2( \nu \| \mu) }{ \delta \epsilon^2 }$   \cite[Proposition 9]{metelli2020importance}. 

Our results yield an improved finite-sample analysis of the SNIS estimator in terms of general $f$-divergences, applicable even when the $\chi^2$-divergence is infinite.

\begin{theorem}[Self-Normalized Importance Sampling]  \label{thm:snis}
    Let $\mu, \nu$ be measures and $g$ be a bounded, positive function.
    Then the self-normalized importance sampling estimator $\hat{g}_{\mathrm{SNIS}}$ satisfies $\abs{\hat{g}_{\mathrm{SNIS}} - \nu_g } \leq \epsilon \cdot \nu_g$ with probability at least $1 - \delta$, as long as $n \gtrsim \nicefrac{  M_{\epsilon, \delta}  }{\epsilon}$,
    where $M_{\epsilon, \delta}$ satisfies $\icov[M_{\epsilon, \delta}]{ g \nu }{\mu} \leq \frac{\epsilon \delta }{6}$ and $\icov[M_{\epsilon, \delta}]{ \nu }{\mu} \leq \frac{\epsilon \delta }{6}$.
\end{theorem}

Note that the requirement on the integrated coverage of both the weighted target distribution $g \nu$ and the target distribution $\nu$ is necessary since the SNIS estimator is a ratio of two importance sampling estimators where the denominator estimates the normalizing constant $Z$.

\section{Proof Techniques for Upper Bounds}\label{sec:proof_ideas}

In this section, we outline some of the key proof techniques used to establish our main upper bounds, with a focus on technical results that may be of independent interest.  We defer full proofs  of all results to the appendix.

\subsection{Upper Bounds on Sample Complexity}\label{ssec:proof_ub}
Our upper bound relies on the same general proof strategy and the same estimator, namely the \emph{median-of-means} \citep{lugosi2019mean, alon1996space}
Formally, given $n$ samples from $\mu$, we partition them into $k$ groups of size $m = \lfloor \nicefrac{n}{k} \rfloor$ and compute the sample mean within each group, i.e., for each group $j \in [k]$, we let
\begin{align}\label{eq:zhat}
    \Zhat_j = \frac{1}{m} \sum_{i=1}^m \lambda(X_i^{(j)}) \quad \text{and} \quad \Zhat = \Med\left( \Zhat_1, \dots, \Zhat_k \right),
\end{align}
where $\Med(\cdot)$ denotes the median operator.  Using standard analysis (cf. \Cref{lem:median_of_means}), it suffices to show that each group mean $\Zhat_j$ is within a constant factor of the true partition function $Z$ with constant probability, say $\nicefrac 23$.  Thus, the main technical challenge is to analyze the concentration of the sample mean $\Zhat_j$ around $Z$.  While we defer full proofs of \Cref{thm:ub_coverage_coverage,thm:ub_fdiv} to \Cref{sec:proof_ub_coverage,app:ub_coverage_proof}, we sketch the main ideas below.

\paragraph{Proof Sketch of \Cref{thm:ub_coverage_coverage}.}
    The first natural approach to analyzing the concentration of $\Zhat_j$ would be to apply standard concentration inequalities such as Chernoff bounds. 
    Unformately as discussed earlier, this approach fails because the density ratio $\nicefrac{d\nu}{d\mu}$ may be heavy-tailed, and thus the variance of the summands may be infinite.
    In order to overcome this challenge, we look at the concentration of truncated versions of the density ratio; note that this truncation purely an analytical tool and the estimator $\Zhat_j$ does not involve any truncation.  We thus analyze the `bulk' of the density ratios and the `tail' separately.  To lower bound $\Zhat_j$, we can ignore the tail entirely, while to upper bound the estimate, we apply Markov's inequality to the definition of coverage.

    Thus, much of the effort arises in controlling the average of the truncated density ratios, whose expectation can be appropriately controlled via the definition of coverage.  One might at first like to apply a standard Chernoff bound, because truncated at $M$ the summands are bounded, but this could at best achieve a sample complexity scaling as $M/\epsilon^2$ where $M$ is such that $\Cov_M(\nu \| \mu) \leq \epsilon$, which is loose.  Instead, we apply Bernstein's inequality and bound the \emph{variance} of the truncated density ratio, which we show can be controlled by the integrated coverage itself (\Cref{lem:variance_ub}).  This interesting \emph{self-normalization} type property that relates the variance of the truncated density ratio to the bias introduced by truncation is the key technical tool that allows us to improve the dependence on $\epsilon$ in the sample complexity. \hfill$\blacksquare$

We can then use this result to prove \Cref{thm:ub_fdiv} by relating the integrated coverage to the $f$-divergence, as we sketch below.
\paragraph{Proof Sketch of \Cref{thm:ub_fdiv}.}
    In order to get a sample complexity that depends on the $f$-divergence, we need to relate the integrated coverage to the $f$-divergence. 
    At first glance, this seems simple since an application of Markov's inequality directly relates the two quantities (\Cref{lem:cov_fdiv}).
    Unfortunately, this direct relationship is not sufficient to get us the desired bound on the integrated coverage. 
    Towards this end, we need to use the growth properties of the function $f$. In particular, since we are interested primarily in the subquadratic regime, we use the assumption that $t^{-2} \cdot f(t)$ is decreasing for $t \geq c$ to relate the integrated coverage to the $f$-divergence as shown in \Cref{lem:fdiv_second_moment}. \hfill$\blacksquare$

\subsection{Asymmetric Estimation via Coverage}\label{ssec:proof_ub_fdiv}

In order to obtain improved dependence on $\epsilon$, in \Cref{thm:ub_lowertail} we relaxed the symmetric estimation requirement to only tightly controlling the lower tail while allowing a looser constant factor upper tail.  The technical approach is built on the following lemma which can be read as a generalization of the Paley-Zygmund inequality to $f$-divergences, whose proof can be found in \Cref{app:paley_zygmund}.
\begin{lemma}\label{lem:gen_paley_zygmund}
    Let $\mu, \nu$ be probability measures on $\cX$ and let $f$ be a convex function as in \Cref{def:fdiv}.  Then for any $0 < \epsilon, u < 1$, it holds that
    \begin{align}
        \pp_{X \sim \mu}\left( \frac{d\nu}{d\mu}(X) \geq 1 - \epsilon \right) \geq \nicefrac{(1-u)\epsilon}{M} \quad \text{where } \Cov_M\left( \nu \| \mu \right) \leq u \cdot \epsilon.
    \end{align}
    In particular, $\pp_{X \sim \mu}\left( \nicefrac{d\nu}{d\mu}(X) \geq 1 - \epsilon \right) \geq \sup_{0 < u  < 1} \nicefrac{(1-  u)\cdot \epsilon}{\gamma_f\left( \nicefrac{\divf{\nu}{\mu}}{u \cdot \epsilon} \right)}$.
\end{lemma}
While \Cref{lem:gen_paley_zygmund} is phrased in terms of $f$-divergences and likelihood ratios, it could otherwise be stated for arbitrary nonnegative random variables, providing a lower bound on the probability that such a variable exceeds a $(1 - \epsilon)$ fraction of its mean.  Indeed, taking $f(t) = (t - 1)^2$ and letting $Y \geq 0$, we see that we recover Paley-Zygmund up to a factor of $4$.  We similarly recover the standard generalization of Paley-Zygmund to higher moments by taking $f(t) = t^p - 1$ for $p > 1$.  In the case that we only assume a $\abs{X}\log(\abs{X})$ moment, we derive a bound that, to the best of our knowledge, is novel, showing that a nonnegative random variable exceeds $(1-\epsilon)$ times its mean with probability at least $\epsilon \cdot \exp\left( - O(1) \cdot \nicefrac{\ee[X \log(X)]}{\epsilon \cdot \ee\left[ X \right]} \right)$.  We now sketch the proof of \Cref{thm:ub_lowertail} using \Cref{lem:gen_paley_zygmund}; full details can be found in \Cref{sec:ub_lowertail_proof}.
\paragraph{Proof Sketch of \Cref{thm:ub_lowertail}.}
    Unlike the previous upper bounds, we do not use a median-of-means estimator here, instead simply taking the $\left( 1 - \alpha \right)$-quantile of the samples for appropriately chosen $\alpha$.  The key idea is to use \Cref{lem:gen_paley_zygmund} to lower bound the probability that a single sample exceeds $(1 - \epsilon) Z$, which then allows us to control the lower tail of the quantile estimator via a Chernoff bound.  We then control the upper tail of the estimator via the definition of coverage to ensure that the quantile does not exceed $M \cdot Z$ with high probability.  \hfill$\blacksquare$

\section{Discussion and Related Work}\label{sec:discussion}

In this work, we have provided a complete characterization of the sample complexity of partition function estimation in terms of integrated coverage and $f$-divergences between the target and proposal distributions.  
We now discuss related work.

\paragraph{Partition Function Estimation.}  Partition function estimation is a classical problem with roots dating back to Gibbs and Boltzman in statistical mechanics, with many classical techniques developed and analyzed in the intervening years \citep{neal2001annealed,gelman1998simulating,jerrum2003counting,owen2013monte}.  One common line of work in theoretical computer science and statistical mechanics assumes some kind of discrete structure and constructs sophisticated annealing schedules and samplers in order to estimate partition functions of complex models such as the Ising model or counting problems such as the permanent of a matrix \citep{jerrum2003counting,bezakova2008accelerating,huber2015approximation,kolmogorov2018faster,stefankovic2009adaptive}.  Another line of work, especially in scientific and Bayesian applications, focuses on continuous distributions in Euclidean space and makes use of strong regularity assumptions in order to ensure quality \citep{chehab2023provable,grosse2013annealing,liu2015estimating}.  In contradistinction to those works, we focus on a minimal assumption setting where we characterize the difficulty of partition function estimation in terms of purely information theoretic quantities, without making any structural or regularity assumptions.

\paragraph{Sampling and Importance Sampling.}  In parallel to partition function estimation, the closely related problem of approximate sampling from a target distribution $\nu$ given sample access to a proposal distribution $\mu$ and the ability to evaluate the density ratio $\nicefrac{d\nu}{d\mu}$ in order to construct mean estimates under target distributions has also been extensively studied \citep{knuth1976mathematics,cappe2005inference,owen2013monte}.  Most relevant to our work is those of \citet{devroye2016sub,chatterjee2018sample}.  The former, while not explicitly about partition function estimation, considers heavy-tailed mean estimation under polynomial moment constraints; our $f$-divergence recover their results when considering Renyi divergences.  On the other hand, \citet{chatterjee2018sample} consider importance sampling and examine the sample complexity under KL constraints, proving that exponentially many samples in the KL divergence are necessary and sufficient for importance sampling to succeed, but do not consider more general $f$-divergences or coverage.  Our results generalize and sharpen theirs, recovering their bounds as a special case and unifying these paradigms.  More recently, \citet{block2023sample} studied the problem of approximate rejection sampling, which is closely related to our results, except they assume \emph{normalized} density ratio access.  On the other hand, \citet{li2018strong} consider sampling with unnormalized density ratios but assume bounded density ratios, whereas \citet{flamich2024some} generalizes this to a bounded KL constraint.  We apply the same algorithm, introduced in \citet{maddison2014sampling}, and generalize these prior results to arbitrary $f$-divergences and coverage constraints.

\bibliographystyle{plainnat}
\bibliography{refs}

\tableofcontents

\appendix

\crefalias{section}{appendix} %

\section{Technical Lemmata}\label{app:technical_lemmata}
In this section, we collect some technical lemmata used in the proofs of our main results.  We begin with two concentration inequalities, before proving \Cref{lem:cov_fdiv} relating coverage to $f$-divergences.

\subsection{Concentration Inequalities}
We first state the standard Chernoff bound (cf. for example \citep{boucheron2013concentration}).
\begin{lemma}[Chernoff Bound]\label{lem:chernoff}
    Let $X_1, \dots, X_m$ be independent random variables such that for each $i \in [m]$, it holds that $X_i \in [0, M]$ almost surely and $\mu = \ee[X_i]$.  Then for any $1 \geq t > 0$ if holds that
    \begin{align}
        \pp\left( \frac 1m \sum_{ i = 1}^m X_i \leq (1 - t) \mu \right) \leq e^{- \frac{m \mu t^2 }{2 M}}.
    \end{align}
    Moreover, for any $t > 0$, it holds that
    \begin{align}
        \pp\left( \frac 1m \sum_{i = 1}^m X_i \geq (1 + t) \mu \right) \leq \exp\left( - \nicefrac{m \mu}{M} \cdot \left( (1 + t)\log(1 + t) - t \right) \right).
    \end{align}
\end{lemma}
Because we are generally considering $t \lesssim \epsilon \ll 1$, we will often simplify the upper tail bound using the fact that $(1 + t) \log(1 + t) - t \geq \nicefrac{t^2}{3}$ for $t \in (0,1)$ so that 
\begin{align}
    \pp\left( \abs{\frac 1m \sum_{ i = 1}^m X_i - \mu} \geq t \mu \right) \leq 2 \exp\left( - \nicefrac{m \mu t^2}{3 M} \right)
\end{align}
for $t \in (0,1)$.

We also make use of the classical Bernstein's inequality (cf. for example \citep{boucheron2013concentration}), which we state here for completeness.
\begin{lemma}[Bernstein's Inequality]\label{lem:bernstein}
    Let $\xi_1, \dots, \xi_m$ be independent random variables such that for each $i \in [m]$, the following properties hold:
    \begin{enumerate}
        \item The variables are centred, i.e. $\ee[\xi_i] = 0$.
        \item The variables are bounded almost surely, i.e. $\abs{\xi_i} \leq M$.
        \item The variables have variance at most $\sigma^2$, i.e. $\ee[\xi_i^2] - \ee\left[ \xi_i \right]^2 \leq \sigma^2$.
    \end{enumerate}
    Then it holds for any $0 < \delta < 1$ that with probability at least $1 - \delta$,
    \begin{align}
        \abs{\frac 1m \sum_{i = 1}^m \xi_i} \leq \sqrt{\frac{2 \sigma^2 \log\left( \nicefrac{2}{\delta} \right)}{m}} + \frac{2 M \log\left( \nicefrac{2}{\delta} \right)}{3 m}.
    \end{align}
\end{lemma}

\section{Proof of \Cref{thm:ub_coverage_coverage}} \label{sec:proof_ub_coverage}

\subsection{Choice of Estimator}\label{app:median_of_means}
In this section, we define the estimator we will use to prove our upper bounds. Recall that the median of $k$ samples $Y_1, \dots, Y_k$ is defined to be $y = \Med\left(\{Y_1, \dots, Y_k\}\right)$ such that $y \in \left\{ Y_1, \dots, Y_k \right\}$ satisfying
\begin{align}
    \abs{\left\{ j \in [k] | Y_j \leq y \right\}} \geq \frac k 2 \quad \text{and} \quad \abs{\left\{ j \in [k] | Y_j \geq y \right\}} \geq \frac k 2.
\end{align}
We will use a median of means estimator to prove the upper bound.  Let $k = \lceil 8 \log(\nicefrac{1}{\delta}) \rceil$ and partition the $n$ samples into $k$ groups of size $m = \lfloor \nicefrac{n}{k} \rfloor$.  Recall that we define our estimator such that for each group $j \in [k]$, we let
\begin{align}%
    \Zhat_j = \frac{1}{m} \sum_{i=1}^m \lambda(X_i^{(j)}) \quad \text{and} \quad \Zhat = \Med\left( \Zhat_1, \dots, \Zhat_k \right).
\end{align}
The key fact about medians we will use is the following standard lemma, which allows us to convert constant probability bounds into high probability bounds using the median of means technique.  
This lemma is typically used in the context of estimation with bounded variance, but we state it here in a more general form for completeness.

\begin{lemma}\label{lem:median_of_means}
    Let $Y_1, \dots, Y_k$ be independent random variables such that for each $i \in [k]$, it holds that $\pp\left( Y_i \in (a, b) \right) \geq \nicefrac 12 + c$ for some $c > 0$.  Then it holds that
    \begin{align}
        \pp\left( \Med\left( Y_1, \dots, Y_k \right) \in (a, b)\right) \geq 1 - \exp\left( -2 k c^2 \right).
    \end{align}
\end{lemma}
\begin{proof}
    Let $y = \Med\left(\{Y_1, \dots, Y_k\}\right)$ and let $X_i = \ind{Y_i \in (a,b)}$.  Note that $\ee[X_i] \geq \nicefrac 12 + c$ for all $i \in [k]$ and let $S = \sum_{i=1}^k X_i$.  Note that if $y \notin (a,b)$, then it must hold that $S \leq \nicefrac k 2$.  Thus we have
    \begin{align}
        \pp\left( y \not\in (a, b) \right) \leq \pp\left( S \leq \nicefrac k2 \right) \leq \pp\left( \mathrm{Bin}\left( \nicefrac{k}{2} + c, k \right) \leq \nicefrac k2 \right) \leq e^{- 2 k c^2}
    \end{align}
    where the last inequality follows from a standard Chernoff bound (\Cref{lem:chernoff}).
\end{proof}

\subsection{Analysis of Sample Mean}

Our upper bound proofs will proceed by showing that $\Zhat_j$ is $\epsilon$-close to $Z$ with constant probability for each $j \in [k]$, and then conclude by applying \Cref{lem:median_of_means}.
Further, since our estimator is scale invariant and the guarantees are multiplicative, we can analyze the case where $Z = 1$ without loss of generality.

The first key lemma we need involves controlling the variance of truncated versions of the density ratio $\frac{d\nu}{d\mu}$ in terms of integrated coverage. 

\begin{lemma} \label{lem:variance_ub}
    For any two probability measures $\mu, \nu$, we have that 
    \begin{align}
         \mathbb{E}_{X \sim \mu} \left[ \left(  \frac{d \nu }{d \mu}(X)\right)^2 \mathbb{I}\left[ \frac{d \nu}{d \mu}(X) \leq M \right] \right] \leq \icov{\nu}{\mu} 
    \end{align}
\end{lemma}
\begin{proof}
    We use a standard tail integration argument.  Indeed, we have
    \begin{align}
        \mathbb{E}_{X \sim \mu} \left[ \left(  \frac{d \nu }{d \mu}(X)\right)^2 \mathbb{I}\left[ \frac{d \nu}{d \mu}(X) \leq M \right] \right] &= \mathbb{E}_{X \sim \nu} \left[   \frac{d \nu }{d \mu}(Y)\mathbb{I}\left[ \frac{d \nu}{d \mu}(Y) \leq M \right] \right] \\
        &= \int_0^\infty \pp_{Y \sim \nu}\left( \frac{d\nu}{d\mu}(Y) \cdot \bbI\left[ \frac{d\nu}{d\mu}(Y) \leq M\right] \geq t \right) dt \\
        &= \int_0^M \pp_{Y \sim \nu}\left( t \leq \frac{d\nu}{d\mu}(Y) \leq M \right) d t \\
        &= \int_0^M \left( \covm[t]{\nu}{\mu} - \covm[M]{\nu}{\mu} \right) dt \\
        &= \int_0^M \covm[t]{\nu}{\mu} dt - M \cdot \covm[M]{\nu}{\mu}.
    \end{align}
    The result follows.

\end{proof}

\newcommand{\invcov}[3][\epsilon]{\mathrm{IC}_{#1}\left( #2 ||  #3 \right)}

Next, we prove a concentration result for the empirical average of the density ratio in terms of integrated coverage. 
The key idea is to note that even though the density ratio may have unbounded variance, we can analyze truncated versions of the density ratio to control the variance. 

\begin{lemma}[Sample Mean Concentration via Integrated Coverage]\label{thm:ub_sample_mean_icov}
    For any two probability measures $\mu, \nu$, let $X_1, \dots, X_m$ be i.i.d. samples from $\mu$.  
    Suppose $M_{\epsilon, \delta}> 0$ is defined to satisfy
    \begin{align}
        \icov[M_{\epsilon , \delta}]{\nu}{\mu} \leq \epsilon \delta M_{\epsilon, \delta},
    \end{align}  
    with the convention that $M_{\epsilon, \delta} = \infty$ if no such finite value exists.
    Then for any $0 < \epsilon, \delta < 1$, it holds that as long as
    \begin{align}
        m \geq \frac{ M_{\epsilon, \delta} \cdot \log(1/\delta) }{\epsilon} ,
    \end{align}
    then with probability at least $1 - \delta$ it holds that
    \begin{align}
        1 - \epsilon\leq \frac 1m \sum_{i = 1}^m \frac{d\nu}{d\mu}(X_i) \leq 1+ \epsilon .
    \end{align}
\end{lemma}

\begin{proof}
    Fix an $M$ to be chosen later and let $\eta_i = \frac{d\nu}{d\mu}(X_i) \cdot \ind{\frac{d\nu}{d\mu}(X_i) \leq M}$ for $i \in [m]$.
    It is immediate that $\ee\left[ \eta_i \right] \leq 1$.  
    Further, from \Cref{def:coverage}, we have that 
    \begin{align}\label{eq:expectation_lower_bound_1}
        \ee\left[ \eta_i \right] \geq 1 - \covm[M]{\nu}{\mu}
    \end{align}
    By \Cref{lem:variance_ub}, it holds that
    \begin{align}
        \ee\left[ \eta_i^2 \right] \leq \icov[M]{\nu}{\mu}
    \end{align}
    and clearly $0 \leq \eta_i \leq M$ almost surely.  
    Thus by \Cref{lem:bernstein}, it holds that with probability at least $1 - \delta$,
    \begin{align}\label{eq:bernstein_application_1}
        \abs{\frac 1m \sum_{i = 1}^m \eta_i - \ee\left[ \eta_i \right]} \leq \sqrt{\frac{2 \icov[M]{\nu}{\mu} \log\left( \nicefrac{2}{\delta} \right)}{m}} + \frac{2 M \log\left( \nicefrac{2}{\delta} \right)}{3 m}.
    \end{align}
    We now analyze the upper and lower tails separately.

    \paragraph{Lower Tail.}  Combining \eqref{eq:bernstein_application_1} and \eqref{eq:expectation_lower_bound_1}, we have that with probability at least $1 - \delta$,
    \begin{align}
        \frac 1m \sum_{i = 1}^m \frac{d\nu}{d\mu}(X_i) &\geq \frac 1m \sum_{i = 1}^m \eta_i \\
        &\geq 1 - \covm[M]{\nu}{\mu} - \sqrt{\frac{2 \icov[M]{\nu}{\mu} \log\left( \nicefrac{2}{\delta} \right)}{m}} - \frac{2 M \log\left( \nicefrac{2}{\delta} \right)}{3 m}.
    \end{align}
    Applying Young's inequality, which says that for $a,b, \lambda \geq 0$ it holds that $\sqrt{ab} \leq \nicefrac a\lambda + \lambda b$, we see that
    \begin{align}
        \sqrt{\frac{2 \icov[M]{\nu}{\mu} \log\left( \nicefrac{2}{\delta} \right)}{m}} &\leq \frac{\icov[M]{\nu}{\mu}}{ M} + \frac{ M \log\left( \nicefrac{2}{\delta} \right)}{2 m}.
    \end{align}
    This gives that with probability at least $1 - \delta$,
    \begin{align}
        \frac 1m \sum_{i = 1}^m \frac{d\nu}{d\mu}(X_i) &\geq 1 - \covm[M]{\nu}{\mu} - \frac{\icov[M]{\nu}{\mu}}{ M} - \frac{2 M \log\left( \nicefrac{2}{\delta} \right)}{3 m}.
    \end{align}
    Setting, $M$ such that $\covm[M]{\nu}{\mu} \leq \epsilon/3$, $\icov[M]{\nu}{\mu} \leq \epsilon M /3$ and $m$ such that $\frac{2 M \log\left( \nicefrac{2}{\delta} \right)}{3 m} \leq \epsilon /3$, we have that with probability at least $1 - \delta$, 
    \begin{align}
        \frac 1m \sum_{i = 1}^m \frac{d\nu}{d\mu}(X_i) &\geq 1 -  \epsilon.
    \end{align}

    \paragraph{Upper Tail.}  We separate the sample mean as
    \begin{align}
        \frac 1m \sum_{i = 1}^m \frac{d\nu}{d\mu}(X_i) &= \frac 1m \sum_{i = 1}^m \eta_i + \frac{1}{m} \sum_{i = 1}^m \frac{d\nu}{d\mu}(X_i) \cdot \ind{\frac{d\nu}{d\mu}(X_i) > M}
    \end{align}
    and analyze each term separately.  By Markov's inequality, it holds that
    \begin{align}
        \pp\left( \frac{1}{m} \sum_{i = 1}^m \frac{d\nu}{d\mu}(X_i) \cdot \ind{\frac{d\nu}{d\mu}(X_i) > M} > t \right) &\leq \frac{\ee\left[ \frac{d\nu}{d\mu}(X_i) \cdot \ind{\frac{d\nu}{d\mu}(X_i) > M} \right]}{t} \\
        &\leq \frac{\Cov_M\left( \nu \| \mu \right)}{t} . 
    \end{align}
    Thus, as long as $ \covm[M]{\nu}{\mu} \leq \frac{\epsilon \delta }{4} $ it holds that with probability at least $1 - \nicefrac{\delta}{2} $, that
    \begin{align}
        \frac{1}{m} \sum_{i = 1}^m \frac{d\nu}{d\mu}(X_i) \cdot \ind{\frac{d\nu}{d\mu}(X_i) > M} \leq \frac{\epsilon}{2}.
    \end{align}
    We apply \Cref{lem:bernstein} again to control the first term.
      Indeed, it holds that with probability at least $1 - \nicefrac{\delta}{3}$ that
    \begin{align}
        \frac 1m \sum_{i = 1}^m \xi_i &\leq 1 - \covm[M]{\nu}{\mu} + \sqrt{\frac{2 \icov[M]{\nu}{\mu} \log\left( \nicefrac{2}{\delta} \right)}{m}} + \frac{2 M \log\left( \nicefrac{2}{\delta} \right)}{3 m} 
    \end{align}
    As before, applying Young's inequality gives that with probability at least $1 - \nicefrac{\delta}{3}$,
    \begin{align}
        \frac 1m \sum_{i = 1}^m \xi_i &\leq 1 - \covm[M]{\nu}{\mu} + \frac{\icov[M]{\nu}{\mu}}{ M} + \frac{7 M \log\left( \nicefrac{3}{\delta} \right)}{6 m}.
    \end{align}
    Setting, $M$ such that $\covm[M]{\nu}{\mu} \leq \frac{\epsilon}{2}$, $\icov[M]{\nu}{\mu} \leq \frac{\epsilon M}{2}$ and $m$ such that $\frac{7 M \log\left( \nicefrac{3}{\delta} \right)}{6 m} \leq \frac{\epsilon}{2}$, we have that with probability at least $1 - \nicefrac{\delta}{3}$, 
    \begin{align}
        \frac 1m \sum_{i = 1}^m \xi_i &\leq 1 + \epsilon.
    \end{align}
    Combining the two parts of the upper tail and applying a union bound, we get the result. 
\end{proof}

\begin{proof}[Proof of \Cref{thm:ub_coverage_coverage}]
    By the previous lemma, we have that for each $j \in [k]$, as long as
    \begin{align}
        m \geq \frac{ M_{\epsilon, 1/4} \cdot \log(4) }{\epsilon/2} ,
    \end{align}
    it holds that with probability at least $3/4$,
    \begin{align}
        1 - \epsilon \leq \Zhat_j \leq 1 + \epsilon.
    \end{align}
    Applying \Cref{lem:median_of_means} with $c = 1/4$ and $k = \lceil 8 \log(\nicefrac{1}{\delta}) \rceil$ gives that with probability at least $1 - \delta$,
    \begin{align}
        1 - \epsilon \leq \Zhat \leq 1 + \epsilon
    \end{align}
    as required. 
\end{proof}

\section{Relationship between Coverage and $f$-Divergences and the Proof of Theorem \ref{thm:ub_fdiv}} \label{sec:lem_cov_fdiv_proof}

In this appendix we establish some relationships between coverage and $f$-divergences, which we then use to prove \Cref{thm:ub_fdiv} in \Cref{app:ub_coverage_proof}.  We begin with these relationships, culminating in a key relationship (\Cref{lem:fdiv_second_moment}) used in the proof of \Cref{thm:ub_sample_mean_fdiv}, which is in turn necessary to prove our upper bound.

\subsection{Coverage and $f$-Divergences}

We begin by stating a fundamental relationship between coverage and $f$-divergences, which will be useful in proving upper bounds on the sample complexity in terms of $f$-divergences.

\begin{lemma}\label{lem:cov_fdiv}
    For any two probability measures $\mu, \nu$ on $\cX$ and any convex function $f$ as in \Cref{def:fdiv}, it holds for all $M > 1$ that
    \begin{align}
        \covm{\nu}{\mu} \leq \frac{M \cdot \divf{\nu}{\mu}}{f(M)}.
    \end{align}
    Thus if $M \geq \gamma_f(\nicefrac{\divf{\nu}{\mu}}{\epsilon})$ for some $\epsilon \in (0,1)$, then $\covm{\nu}{\mu} \leq \epsilon$.
\end{lemma}

\begin{proof}
     By definition of coverage, we have that
    \begin{align}
        \Cov_M\left( \nu \| \mu \right) &= \pp_{X \sim \nu}\left( \frac{d\nu}{d\mu}(X) \geq M \right) = \ee_{X \sim \mu}\left[ \frac{d\nu}{d\mu}(X) \cdot \ind{\frac{d\nu}{d\mu}(X) \geq M} \right] \\
        &\leq \ee\left[ \left( 1 + \frac{M - 1}{f(M)} \cdot f\left( \frac{d\nu}{d\mu}(X) \right) \right) \cdot \ind{\frac{d\nu}{d\mu}(X) \geq M} \right] \\
        &\leq \pp_{X \sim \mu}\left( \frac{d\nu}{d\mu}(X) \geq M \right) + \frac{M - 1}{f(M)} \cdot \divf{\nu}{\mu},
    \end{align}
    where the first inequality follows from \Cref{lem:convexf} with $x = \frac{d \nu}{d \mu}(X)$ for $X \sim \mu$ and the second inequality follows from the definition of $f$-divergence in \eqref{eq:fdiv} and the fact that $f \geq 0$.
    To conclude, we observe that for $M \geq 1$, it holds that $f$ is monotone increasing and thus we may apply Markov's inequality to obtain
    \begin{align}
        \pp_{X \sim \mu}\left( \frac{d\nu}{d\mu}(X) \geq M \right) \leq \pp_{X \sim \mu}\left( f\left( \frac{d\nu}{d \mu}(X) \right) \geq f(M)\right) \leq \frac{\ee_{X \sim \mu}\left[ f\left( \frac{d\nu}{d\mu}(X) \right) \right]}{f(M)} \leq \frac{\divf{\nu}{\mu}}{f(M)}.
    \end{align}
    Adding this to the previous bound yields the result.
\end{proof}   

\begin{lemma}\label{lem:convexf}
Let $f: \rr_{\geq 0} \to \rr_{\geq 0}$ be a convex function such that $f(1) = f'(1) = 0$ and let $M \geq 1$.  Then it holds that for any $x \geq M$,
\begin{align}
    x \leq 1 + \frac{M - 1}{f(M)} \cdot f(x).
\end{align}
\end{lemma}
\begin{proof}
    Let $t = \frac{M - 1}{x - 1}$; by the assumption that $1 < M \leq x$ we have that $t \in [0,1)$.  Observe that $M = t \cdot 1 + (1 - t) \cdot x$ and so by convexity of $f$, it holds that
    \begin{align}
        f(M) \leq t \cdot f(1) + (1 - t) \cdot f(x) = (1 - t) \cdot f(x) = \frac{M - 1}{x - 1} \cdot f(x).
    \end{align}
    The result follows by rearranging the above inequality.
\end{proof}

We now prove a key relationship between the second moment of the truncated density ratio and the $f$-divergence between the two measures.

\begin{lemma}\label{lem:fdiv_second_moment}
    Let $f$ be a convex function as in \Cref{def:fdiv} such that there exists some $c > 1$ such that for all $M \geq t \geq c$ it holds that
    \begin{align}
        \frac{f(M)}{M^2} \leq \frac{f(t)}{t^2}.
    \end{align}
    Then for any probability measures $\mu, \nu$ on $\cX$ and any $M \geq c$, it holds that
    \begin{align}
        \frac{\icov[M]{\nu}{\mu}}{M} \leq \frac{c^2}{M} + \frac{M \cdot \divf{\nu}{\mu}}{f(M)}.
    \end{align}
\end{lemma}
\begin{proof}
    Note first that if $\nicefrac{f(M)}{M^2} \leq \nicefrac{f(t)}{t^2}$ then
    \begin{align}\label{eq:fdiv_quadratic}
        t^2 \leq \frac{M^2 \cdot f(t)}{f(M)}.
    \end{align}
    We compute:
    \begin{align}
        \ee_{X \sim \mu}\left[ \left( \frac{d\nu}{d\mu}(X) \cdot \ind{\frac{d\nu}{d\mu}(X) \leq M} \right)^2 \right] &= \ee_{X \sim \mu}\left[ \left( \frac{d\nu}{d\mu}(X) \cdot \ind{\frac{d\nu}{d\mu}(X) \leq M} \right)^2  \ind{\frac{d\nu}{d\mu}(X) \leq c}\right] \\
        &\quad + \ee_{X \sim \mu}\left[ \left( \frac{d\nu}{d\mu}(X) \cdot \ind{\frac{d\nu}{d\mu}(X) \leq M} \right)^2  \ind{M \geq \frac{d\nu}{d\mu}(X) > c}\right] \\
        &\leq c^2 + \ee_{X \sim \mu}\left[ \frac{M^2 \cdot f\left( \frac{d\nu}{d\mu}(X) \right)}{f(M)}  \ind{M \geq \frac{d\nu}{d\mu}(X) > c}\right] \\
        &\leq c^2 + \frac{2 M^2 \cdot \divf{\nu}{\mu}}{f(M)},
    \end{align}
    where the first inequality follows from \eqref{eq:fdiv_quadratic} and the second inequality follows the nonnegativity of the integrand.
\end{proof}
Note that this lemma connnects the key quantity $\icov[M]{\nu}{\mu}/M$ used in our upper bounds to the $f$-divergence between the two measures.  In particular, it shows that if we choose $M$ sufficiently large so that $\nicefrac{M \cdot \divf{\nu}{\mu}}{f(M)}$ is small, then $\nicefrac{\icov[M]{\nu}{\mu}}{M}$ is also small.  We are now ready to prove our main upper bound in terms of $f$-divergences.

\subsection{Proof of Theorem~\ref{thm:ub_fdiv}}\label{app:ub_coverage_proof}
 
As before, we will use the median-of-means estimator defined in \eqref{eq:zhat} and thus it suffices to understand the behavior of the sample mean estimators in terms of $f$-divergences.

\begin{lemma}[Sample Mean Concentration via $f$-divergence]\label{thm:ub_sample_mean_fdiv} 
    For any two probability measures $\mu, \nu$, let $X_1, \dots, X_m$ be i.i.d. samples from $\mu$. 
    Let $f$ be a convex function as in \Cref{def:fdiv} such that there exists some $c > 1$ such that for all $M \geq t \geq c$ it holds that
    \begin{align}
        \frac{f(M)}{M^2} \leq \frac{f(t)}{t^2}.
    \end{align}
    Then for any $0 < \epsilon, \delta < 1$, as long as
    \begin{align}
        m \geq \frac{ \gamma_f\left( \nicefrac{\divf{\nu}{\mu}}{\epsilon \delta} \right) \cdot \log(1/\delta) }{\epsilon} \vee \frac{c^2 \cdot \log(1/\delta)}{\epsilon^2 \delta} ,
    \end{align}
    then with probability at least $1 - \delta$ it holds that
    \begin{align}
        1 - \epsilon\leq \frac 1m \sum_{i = 1}^m \frac{d\nu}{d\mu}(X_i) \leq 1+ \epsilon .
    \end{align}
\end{lemma}

\begin{proof}
The first step is the following lemma relating coverage to $f$-divergences, the proof of which we defer to \Cref{sec:lem_cov_fdiv_proof}.

    To apply \Cref{thm:ub_sample_mean_icov}, we need to chose $M$ such that the integrated coverage $\icov[M]{\nu}{\mu} \leq \epsilon \delta M $. 
    From \Cref{lem:fdiv_second_moment}, we have that we need $c^2 / M + M \cdot \divf{\nu}{\mu} / f(M) \leq \epsilon \delta$.
    Recall from the definition of $\gamma_f$ in \eqref{eq:gammaf}, we have if we set $M = \gamma_f(\nicefrac{\divf{\nu}{\mu}}{\epsilon \delta})$, then by definition of $\gamma_f$ it holds that $\frac{M \cdot \divf{\nu}{\mu}}{f(M)} \leq \epsilon \delta$.  
    Thus, we setting $ M \geq \max\left\{ c^2 / \epsilon \delta , \gamma_f\left(\frac{\divf{\nu}{\mu}}{\epsilon \delta}\right) \right\} $ and applying \Cref{thm:ub_sample_mean_icov} completes the proof.
\end{proof}

\section{Proofs of Lower Bounds}\label{app:lower_bound_proof}
In this appendix, we provide complete proofs of our lower bound results.  We first prove the most interesting case, the general lower bound in Theorem~\ref{thm:lb} (Section~\ref{ssec:lower_bound_proof}) for superlinear but subquadratic $f$-divergences.  We proceed with the lower bound in terms of integrated coverage (\Cref{thm:lb_coverage}) in \Cref{ssec:icov_lower_bound} and then the linear lower bound in Proposition~\ref{prop:lb_linear} (Section~\ref{app:lb_linear_proof}); finally we prove the superquadratic lower bound in Proposition~\ref{prop:superquadratic_lb} (Section~\ref{app:superquadratic_lb_proof}).

\subsection{Proof of Theorem~\ref{thm:lb}}\label{ssec:lower_bound_proof}
We first outline the construction and then prove the proposition.  Let $\mu$ be any measure and consider a family of distributions $\nu_{p,\epsilon}$ parameterized by $p \in [0,1]$ and $\epsilon \in (0, \frac{1}{4})$ defined such that
\begin{align}\label{eq:densityratio_nupeps}
    \frac{d\nu_{p,\epsilon}}{d\mu}(X) = \begin{cases}
        1 - \epsilon & \text{w.p. } 1 - p, \\
        1 + \epsilon\left( \frac 1p - 1 \right) & \text{w.p. } p.
    \end{cases}
\end{align} 
We will let $\cV \subset \left\{ \nu_{p,\epsilon} \right\}$ such that for all $\nu \in \cV$, it holds that $\divf{\nu}{\mu} \leq C + f(1 - \epsilon)$.  To help with this objective, we first compute the $f$-divergence between $\nu_{p,\epsilon}$ and $\mu$.
\begin{lemma}\label{lem:fdiv_nupeps_mu}
    Suppose that
    \begin{align}
        \gamma_f\left( \frac{C}{2 \epsilon} \right) \geq 2.
    \end{align}
    Then it holds that $\divf{\nu_{p,\epsilon}}{\mu} \leq C +  f(1 - \epsilon)$ whenever
    \begin{align}
        p \geq \frac{2\cdot\epsilon}{\gamma_f\left( \frac{C}{2 \cdot \epsilon} \right)}.
    \end{align}
\end{lemma}
\begin{proof}
    By definition of the $f$-divergence and \eqref{eq:densityratio_nupeps}, we have that
    \begin{align}
        \divf{\nu}{\mu} &= (1 - p) \cdot f(1 - \epsilon) + p \cdot f\left( 1 + \epsilon\left( \frac 1p - 1 \right) \right) \\
        &\leq f(1 - \epsilon) + p \cdot f\left( 1 - \epsilon + \frac \epsilon p\right).
    \end{align}
    Thus as long as
    \begin{align}\label{eq:lb_fdiv_condition}
        p \cdot f\left( 1 - \epsilon + \frac \epsilon p\right) \leq C,
    \end{align}
    the $f$-divergence condition will be satisfied.

    We first claim that if \eqref{eq:lb_fdiv_condition} holds for $p$, then it holds for all $p < q \leq 1$.  To see this, observe that for such $q$,
    \begin{align}
        q \cdot f\left( 1 - \epsilon + \frac \epsilon q\right) &= \left( q (1 - \epsilon) + \epsilon \right) \cdot \frac{f\left( \frac{q (1 - \epsilon) + \epsilon}{q}\right)}{\frac{q (1 - \epsilon) + \epsilon}{q}} \\
        &\leq \left( p(1 - \epsilon) + \epsilon  \right) \cdot \frac{f\left( \frac{p (1 - \epsilon) + \epsilon}{p} \right)}{\frac{p(1 - \epsilon) + \epsilon}{p}} \\
        &= p \cdot f\left( 1 - \epsilon + \frac \epsilon p\right) \leq C,
    \end{align}
    where the first inequality comes from the fact that $t \mapsto \nicefrac{f(t)}t$ is non-decreasing for $t \geq 1$.

    Now observe that the \eqref{eq:lb_fdiv_condition} holds if and only if
    \begin{align}
        \frac{f\left( \frac{p (1 - \epsilon) + \epsilon}{p} \right)}{\frac{p (1 - \epsilon) + \epsilon}{p} } \leq \frac{C}{p (1 - \epsilon) + \epsilon}.
    \end{align}
    By definition of the $\gamma_f$ function in \eqref{eq:gammaf}, this holds whenever
    \begin{align}
        \frac{p(1 - \epsilon) + \epsilon}{p} \leq \gamma_f\left( \frac{C}{p (1 - \epsilon) + \epsilon} \right),
    \end{align}
    which, after rearranging, is satisfied whenever
    \begin{align}
        p \geq \frac{\epsilon}{\gamma_f\left( \frac{C}{p (1 - \epsilon) + \epsilon} \right) - 1}.
    \end{align}
    By the assumption that $\gamma_f\left( \frac{C}{2 \epsilon} \right) \geq 2$, this last is implied by $p \geq \frac{2 \cdot \epsilon}{\gamma_f\left( \frac{C}{p (1 - \epsilon) + \epsilon} \right)}$.  Note that for $p \leq \epsilon$ it holds that
    \begin{align}
        \frac{2 \cdot \epsilon}{\gamma_f\left( \frac{C}{p (1 - \epsilon) + \epsilon} \right)} \leq\frac{2 \cdot \epsilon}{\gamma_f\left( \frac{C}{2 \cdot \epsilon} \right)},
    \end{align}  
    and thus the result follows.
\end{proof}
Thus for fixed $\epsilon$, let
\begin{align}\label{eq:definition_cV}
    \cV = \left\{ \nu_{p,\epsilon'} \bigg| \frac{2 \epsilon'}{\gamma_f\left( \frac{C}{2 \epsilon'} \right)} \leq p \leq 1, \, 0 \leq \epsilon' \leq \frac 14 \right\}
\end{align}

We can now complete the proof of the proposition.
\begin{proof}[Proof of \Cref{thm:lb}]
    Let $\mu$ be any probability measure and let $\cV$ be defined as in \eqref{eq:definition_cV}.  By \Cref{lem:fdiv_nupeps_mu}, it holds that for all $\nu \in \cV$, we have that $\divf{\nu}{\mu} \leq C + f(1 - \epsilon)$.

    Note that for any fixed $p$, in order to distinguish between distributions $\nu_{0} = \mu$ and $\nu_{p,\epsilon}$, with probability at least $\nicefrac 23$, one must observe at least one sample from the high-density region of $\nu_{p,\epsilon}$, which occurs with probability $p$ under $\mu$.  On the other hand, the probability of observing only low-density samples after $n$ samples is
    \begin{align}
        (1 - p)^n \geq e^{- 2 n p}.
    \end{align}
    Thus for $n \leq \nicefrac{\log(\nicefrac 32)}{2p}$, no estimator can distinguish between $\nu_0$ and $\nu_{p,\epsilon}$ with probability at least $\nicefrac 23$.  On the other hand, the normalizing constants satisfy $Z = \lambda(X_1)$ in the former case and $Z = \nicefrac{\lambda(X_1)}{1 - \epsilon}$ in the latter case, and thus any estimator that achieves $(1 - \epsilon) Z \leq \Zhat \leq (1 + \epsilon) Z$ with probability at least $\nicefrac 23$ must be able to distinguish between these two distributions.  The result follows by setting $p = \frac{2 \cdot \epsilon}{\gamma_f\left( \frac{C}{2 \cdot \epsilon} \right)}$.
\end{proof}

\subsection{Proof of Theorem \ref{thm:lb_coverage}}\label{ssec:icov_lower_bound}
    We use the same construction as in the proof of \Cref{thm:lb}.  Indeed, let $\nu_{p,\epsilon}$ be defined as in \eqref{eq:densityratio_nupeps}.  Observe that for any $\nu_{p,\epsilon}$, it holds that
    \begin{align}
        \covm[M]{\nu_{p,\epsilon}}{\mu} = \begin{cases}
            1 & M < 1 - \epsilon \\
            p & 1 - \epsilon \leq M < 1 + \epsilon\left( \frac 1p - 1 \right) \\
            0 & M \geq 1 + \epsilon\left( \frac 1p - 1 \right)
        \end{cases}
    \end{align}
    Thus it holds that
    \begin{align}
        \icov[M]{\nu_{p,\epsilon}}{\mu} = \begin{cases}
            M & M < 1 - \epsilon \\
            (1 - \epsilon) + p \cdot (M - (1 - \epsilon)) & 1 - \epsilon \leq M < 1 + \epsilon\left( \frac 1p - 1 \right) \\
            1 + \epsilon \left( \frac 1p - 1 \right) & M \geq 1 + \epsilon\left( \frac 1p - 1 \right)
        \end{cases}.
    \end{align}
    By the identical argument as the proof of \Cref{thm:lb}, in order to achieve $(1 \pm \epsilon)$ multiplicative accuracy with probability at least $\nicefrac 23$, one must be able to distinguish $\nu_0 = \mu$ from $\nu_{p,\epsilon}$ with probability at least $\nicefrac 23$, which requires at least $\nicefrac{\log(\nicefrac 32)}{2p}$ samples.  Letting $p = \nicefrac{\epsilon^2}{1 + \epsilon}$, we claim that for $M = \nicefrac{\epsilon}{p}$, it holds that $\icov[M]{\nu_{p,\epsilon}}{\mu} \leq M \cdot \epsilon$.  Indeed, observe that for such $M$, we are clearly are in the intermediate regime above and thus
    \begin{align}
        \frac{\icov[M]{\nu}{\mu}}{M} = \frac{(1 - \epsilon)(1 - p)}{M} + p \leq \frac{p}{4 \cdot \epsilon} + p \leq \epsilon.
    \end{align}
    On the other hand, we have established that at least $\nicefrac{\log(\nicefrac 32)}{2p} = \Theta\left( \nicefrac{M_\epsilon}{\epsilon} \right)$ samples are required to achieve the desired accuracy.  The result follows. \hfill$\blacksquare$

\subsection{Proof of Proposition~\ref{prop:lb_linear}}\label{app:lb_linear_proof}
We prove a slightly restated version of Proposition~\ref{prop:lb_linear}.
\begin{proposition}
    Let $f$ be a convex function as in \Cref{def:fdiv} such that the map $t \mapsto f(t) / t$ is uniformly bounded.  For any $0 < C < f'(\infty)$, there exist measures $\mu, \nu$ such that $\divf{\nu}{\mu} \leq C$ and such that no algorithm given any finite number of independent samples from $\mu$ and oracle access to $\frac{d\nu}{d\mu}(\cdot) \cdot Z$ can produce an estimate $\Zhat$ satisfying a $1 \pm \epsilon$ approximation of $Z$ with probability at least $\nicefrac 34$ for any $\epsilon \leq \nicefrac{f'(\infty)}{(f'(\infty) - C)}$.  
    
    Moreover, for any $C \geq f'(\infty)$, there exist measures $\mu, \nu$ such that $\divf{\nu}{\mu} \leq C$ and such that no algorithm given any finite number of independent samples from $\mu$ and oracle access to $\frac{d\nu}{d\mu}(\cdot) \cdot Z$ can produce an estimate $\Zhat$ satisfying a $1 \pm \epsilon$ approximation of $Z$ with probability at least $\nicefrac 34$ for any $\epsilon > 0$. 
\end{proposition}
\begin{proof}
    Let $\mu = \delta(0)$ and $\nu = \ber(1 - q)$.  Then it holds that
    \begin{align}
        \divf{\nu}{\mu} = q f'(\infty).
    \end{align}
    Thus, for any finite $C > 0$, choosing $q \leq \nicefrac{C}{f'(\infty)}$ ensures that $\divf{\nu}{\mu} \leq C$.  However, as $\mu$ is a point mass at $0$, no finite number of samples from $\mu$ reveals any information about $q$.
    
    In the case where $C \leq f'(\infty)$, we have
    \begin{align}
        1 - \frac{C}{f'(\infty)} \leq q \leq 1.
    \end{align}
    and thus the tightest possible constant factor approximation of $Z$ is at least
    \begin{align}
        \frac{f'(\infty)}{f'(\infty) - C}.
    \end{align}
    In the case where $C > f'(\infty)$, we may choose any $q \in [0, 1]$ and thus no finite number of samples from $\mu$ can yield any constant factor approximation of $Z$.

\end{proof}

\subsection{Proof of Proposition~\ref{prop:superquadratic_lb}}\label{app:superquadratic_lb_proof}

Let $\cX = \left\{ 0, 1 \right\}$ with $\mu_p(0) = 1- p$ and $\mu_p(1) = p$; let $\cU = \left\{ \mu_p : \nicefrac 14 \leq p \leq \nicefrac 12 \right\}$.  Let $\nu = \delta_1$ so
\begin{align}
    \frac{d\nu}{d\mu}(X) = \begin{cases}
        0 & X = 0 \\
        \frac{1}{p} & X = 1
    \end{cases}.
\end{align} 
Thus $\norm{\frac{d\nu}{d\mu}}_{\infty} \leq 4$ for all $\mu \in \cU$ by the lower bound on $p$.  Let $\lambda(X) = \ind{X = 1}$ be an unnormalized density ratio for $\nu$ with respect to $\mu_p$ and observe that estimating $Z$ is equivalent to estimating $p$.  Let $p = \nicefrac 12$ and $p' = \nicefrac 12 - 2 \epsilon$ and consider the two measures $\mu = \mu_p$ and $\mu' = \mu_{p'}$.  In order to estimate $Z$ to within a factor of $(1 \pm \epsilon)$ with probability at least $\nicefrac{2}{3}$ under both $\mu$ and $\mu'$, one must distinguish between the two distributions with probability at least $\nicefrac{1}{3}$.  By standard hypothesis testing lower bounds such as Le Cam's method (cf. e.g. \citet{wainwright2019high}), the probability of error is lower bounded by $\nicefrac 12 \left( 1 - \tv\left( \mu^{\otimes n}, \mu^{' \otimes n} \right) \right)$.  Thus we require $\tv\left( \mu^{\otimes n}, \mu^{' \otimes n} \right)  \geq \nicefrac 13$.  On the other hand, by Pinsker's inequality and the chain rule for KL divergence, it holds that
\begin{align}
    \tv\left( \mu^{\otimes n}, \mu^{' \otimes n} \right) \leq \sqrt{\frac{n}{2} \cdot \dkl{\mu}{\mu'}}.
\end{align}
Thus the result will follow if we can show that 
\begin{align}
    \dkl{\mu}{\mu'} \lesssim \epsilon^2.
\end{align}
We compute that
\begin{align}
    \dkl{\mu}{\mu'} &= \frac 12 \cdot \left( \log\left( \frac{\nicefrac 12}{\nicefrac 12 - 2 \epsilon} \right) + \log\left( \frac{\nicefrac 12}{\nicefrac 12 + 2 \epsilon} \right) \right) \\
    &= \frac 12 \cdot \log\left( \frac{1}{1 - 16 \epsilon^2} \right) \\
    &\leq 8 \epsilon^2,
\end{align}
where the last inequality follows from the fact that $\log\left( \nicefrac{1}{1 - x} \right) \leq 2 x$ for all $x \in (0, \nicefrac 12)$.  This completes the proof. \hfill$\blacksquare$

\section{Proof of Theorem~\ref{thm:ub_lowertail}}\label{sec:ub_lowertail_proof}
In this appendix, we prove \Cref{thm:ub_lowertail}.  We begin by proving the Generalized Paley-Zygmund inequality, before applying it to demonstrate that for appropriately chosen $\alpha$, the empirical $\left( 1 - \alpha \right)$-quantile of the samples has the stated performance guarantees with high probability.

\subsection{Generalized Paley-Zygmund}\label{app:paley_zygmund}
We begin by stating a slightly stronger result than the Generalized Paley-Zygmund inequality in \Cref{lem:gen_paley_zygmund}.
\begin{lemma}\label{lem:gen_paley_zygmund_stronger}
    Let $\mu, \nu$ be probability measures on $\cX$.
    Then for any $0 < \epsilon < 1$, it holds for any $0 < u < 1$ that 
    \begin{align}
        \pp_{X \sim \mu}\left( \frac{d\nu}{d\mu}(X) \geq 1 - \epsilon \right) \geq \frac{(1-  u)\cdot \epsilon}{M} \quad \text{where } \Cov_M\left( \nu \| \mu \right) \leq u \cdot \epsilon.
    \end{align}
\end{lemma}
\begin{proof}
    Let $Y = \frac{d\nu}{d\mu}(X)$ for $X \sim \mu$.  By definition, we have that $\ee[Y] = 1$ and thus
    \begin{align}
        1 = \ee\left[ Y \right] \leq 1 - \epsilon + \ee\left[ Y \cdot \ind{Y \geq 1 - \epsilon} \right].
    \end{align}
    Rearranging, we have for any $M$ that
    \begin{align}
        \epsilon &\leq \ee\left[ Y \cdot \ind{Y \geq 1 - \epsilon} \right] \\
        &= \ee\left[ Y \cdot \ind{1 - \epsilon \leq Y < M} \right] + \ee\left[ Y \cdot \ind{Y \geq M} \right] \\
        &\leq M \cdot \pp\left( Y \geq 1 - \epsilon \right) + \ee\left[ Y \cdot \ind{Y \geq M} \right] \\
        &= M \cdot \pp\left( Y \geq 1 - \epsilon \right) + \Cov_M(\nu \| \mu).
    \end{align}
    The result follows by rearranging and using the assumption that $\Cov_M(\nu \| \mu) \leq u \cdot \epsilon$.
\end{proof}
We now derive \Cref{lem:gen_paley_zygmund} from \Cref{lem:gen_paley_zygmund_stronger} by applying \Cref{lem:cov_fdiv}.
\begin{proof}[Proof of \Cref{lem:gen_paley_zygmund}]
    The first part of the statement is just \Cref{lem:gen_paley_zygmund_stronger}.  For the second part, by \Cref{lem:cov_fdiv}, it holds that for
    \begin{align}
        M \geq \gamma_f\left( \frac{\divf{\nu}{\mu}}{u \cdot \epsilon} \right),
    \end{align}
    we have that $\Cov_M(\nu \| \mu) \leq u \cdot \epsilon$.  The result follows by plugging this into \Cref{lem:gen_paley_zygmund_stronger} and optimizing over $u$.
\end{proof}

\subsection{Completing the Proof of Theorem~\ref{thm:ub_lowertail}}

We restate the result with constants included.
\begin{theorem}\label{thm:ub_other_cov}
    Let $\mu, \nu$ be probability measures on $\cX$ and suppose $X_1, \dots, X_n \sim \mu$ are independent with $\lambda$ an unnormalized density ratio. Then there is an estimator $\Zhat$ such that for any $0 < \epsilon < 1$, if
    \begin{align}
        n \geq \frac{18 M}{\epsilon} \cdot \log\left( \nicefrac{2}{\delta} \right) \quad \text{where } \Cov_M\left( \nu \| \mu \right) \leq \nicefrac \epsilon 4,
    \end{align}
    it holds that with probability at least $1 - \delta$ that $(1 - \epsilon) \cdot Z \leq \Zhat \leq M \cdot Z$.
\end{theorem}

We first define the estimator, motivated by \Cref{lem:gen_paley_zygmund_stronger}.  Indeed for arbitrary $u, t$, we will choose some $M$ such that $\Cov_M(\nu \| \mu) < u \cdot \epsilon$ and let $\alpha = \frac{(1 - t)(1 - u) \cdot \epsilon}{M}$.  We then define the estimator as the $\left( 1 - \alpha \right)$-quantile of the samples $\lambda(X_1), \dots, \lambda(X_n)$; that is, letting $\lambda_{(1)} \leq \lambda_{(2)} \leq \dots \leq \lambda_{(n)}$ be the order statistics of the samples, we define $\Zhat = \lambda_{(\lceil (1 - \alpha) n \rceil)}$.  We now analyze the quality of this estimator when $t = u = \nicefrac 12$.
\begin{proof}[Proof of \Cref{thm:ub_other_cov}]
    We analyze each tail separately and then apply a union bound to conclude.
    
    \paragraph{Lower Tail.}  Let $\eta_i = \ind{\frac{d\nu}{d\mu}(X_i) \geq 1 - \epsilon}$ for $i \in [n]$.  By \Cref{lem:gen_paley_zygmund_stronger} it holds for $M$ such that $\Cov_M(\nu \| \mu) \leq u \cdot \epsilon$ that
    \begin{align}
        \ee[\eta_i] = \pp_{X \sim \mu}\left( \frac{d\nu}{d\mu}(X) \geq 1 - \epsilon \right) \geq \frac{(1- u) \cdot \epsilon}{M}.
    \end{align}
    Applying a standard Chernoff bound (\Cref{lem:chernoff}), we have that
    \begin{align}
        \pp\left( \frac 1n \sum_{i = 1}^n \eta_i \leq \frac{(1 - t)(1 - u) \cdot \epsilon}{M} \right) \leq \exp\left( - \frac{n t^2 (1 - u) \cdot \epsilon}{2 M} \right).
    \end{align}
    Now observe that by definition of the quantile,
    \begin{align}
        \pp\left( \Zhat \leq (1 - \epsilon) Z \right) \leq \pp\left( \frac 1n \sum_{i = 1}^n \eta_i \leq \alpha \right) \leq \exp\left( - \frac{n t^2 (1 - u) \cdot \epsilon}{2 M} \right).
    \end{align}
    Thus by setting
    \begin{align}\label{eq:other_lower_tail_n}
        n \geq \frac{2 M}{t^2 (1 - u) \cdot \epsilon} \cdot \log\left( \nicefrac{2}{\delta} \right),
    \end{align}
    we have with probability at least $1 - \nicefrac{\delta}{2}$ that $\Zhat \geq (1 - \epsilon) Z$.

    \paragraph{Upper Tail.} Fix $M' \geq 1$ and let $\xi_i = \ind{\frac{d\nu}{d\mu}(X_i) \geq M'}$ for $i \in [n]$.  By \Cref{lem:cov_under_mu}, it holds that
    \begin{align}
        \ee[\xi_i] = \pp_{X \sim \mu}\left( \frac{d\nu}{d\mu}(X) \geq M' \right) \leq \frac{\Cov_{M'}(\nu \| \mu)}{M'}.
    \end{align}
    Thus by a standard Chernoff bound (\Cref{lem:chernoff}), we have that for $0 < s \leq 1$,
    \begin{align}
        \pp\left(\frac 1n \sum_{i = 1}^n \xi_i \geq (1 + s) \frac{\Cov_{M'}\left( \nu \| \mu \right)}{M'}  \right) \leq \exp\left( - \frac{n s^2 \cdot \Cov_{M'}(\nu \| \mu)}{3M'} \right).
    \end{align}
    It then follows that as long as
    \begin{align}\label{eq:other_upper_tail_n}
        n \geq \frac{3 M'}{s^2 \cdot \Cov_{M'}(\nu \| \mu)} \cdot \log\left( \nicefrac{2}{\delta} \right) \quad \text{and} \quad (1 + s) \frac{\Cov_{M'}(\nu \| \mu)}{M'} \leq \alpha,
    \end{align}
    we have that with probability at least $1 - \nicefrac{\delta}{2}$, that $\Zhat \leq M' Z$.
    
    \paragraph{Concluding the Proof.}  We now set $t = u = \nicefrac 12$ so that $\alpha = \frac{\epsilon}{4 M}$. Thus \eqref{eq:other_lower_tail_n} implies that as long as $n \geq \nicefrac{16 M}{\epsilon} \cdot \log\left( \nicefrac{2}{\delta} \right)$ with $\Cov_M(\nu \| \mu) \leq \nicefrac{\epsilon}{2}$, we have with probability at least $1 - \nicefrac{\delta}{2}$ that $\Zhat \geq (1 - \epsilon) Z$.  Moreover, setting $s = \nicefrac 12$, it follows from \eqref{eq:other_upper_tail_n} that as long as
    \begin{align}
        n \geq \frac{12}{\nicefrac{2\alpha}3} \cdot \log\left( \nicefrac{2}{\delta} \right) = \frac{18 M}{\epsilon} \cdot \log\left( \nicefrac{2}{\delta} \right)
    \end{align}
    that with probability at least $1 - \nicefrac{\delta}{2}$, it holds that $\Zhat \leq M' \cdot Z$ for any $M'$ such that $\Cov_{M'}(\nu \| \mu) \leq \nicefrac{\epsilon}{4}$.  The result follows by applying a union bound.
\end{proof}

We now prove a necessary technical lemma used in the proof of \Cref{thm:ub_other_cov} above.
\begin{lemma}\label{lem:cov_under_mu}
    Let $\mu, \nu$ be probability measures on $\cX$ and let $M \geq 1$.  Then it holds that
    \begin{align}
        \pp_{X \sim \mu}\left(\frac{d\nu}{d \mu}(X) \geq M \right) \leq \frac{\Cov_M(\nu \| \mu)}{M}.
    \end{align}
\end{lemma}
\begin{proof}
    We compute
    \begin{align}
        \pp_{X \sim \mu}\left( \frac{d\nu}{d \mu}(X) \geq M \right) &= \ee_{X \sim \mu}\left[ \ind{ \frac{d\nu}{d \mu}(X) \geq M } \right] \\
        &= \ee_{Y \sim \nu}\left[ \frac{1}{\frac{d\nu}{d \mu}(Y)} \cdot \ind{ \frac{d\nu}{d \mu}(Y) \geq M } \right] \\
        &\leq \frac{1}{M} \cdot \ee_{Y \sim \nu}\left[ \frac{d\nu}{d \mu}(Y) \cdot \ind{ \frac{d\nu}{d \mu}(Y) \geq M } \right] \\
        &= \frac{\Cov_M(\nu \| \mu)}{M}.
    \end{align}
\end{proof}

We now conclude this section by deriving \Cref{thm:ub_lowertail} from \Cref{thm:ub_other_cov}.
\begin{proof}[Proof of \Cref{thm:ub_lowertail}]
    By \Cref{lem:cov_fdiv}, it holds that for
    \begin{align}
        M \geq \gamma_f\left( \frac{4 \cdot \divf{\nu}{\mu}}{\epsilon} \right),
    \end{align}
    we have that $\Cov_M(\nu \| \mu) \leq \nicefrac{\epsilon}{4}$.  The result follows by plugging this into \Cref{thm:ub_other_cov}.
\end{proof}

\section{Proof of Proposition~\ref{prop:sampling}}\label{sec:sampling_proof}

We use the classical $A^*$ sampling algorithm~\citep{li2018strong} which is known to produce approximate samples from the target distribution in the case where the Radon-Nikodym derivative is uniformly bounded or more generally under finite $\dkl{\nu}{\mu}$ \citep{flamich2024some}. The algorithm proceeds by first sampling $X_1, \dots, X_n \sim \mu$ and $E_1, \dots, E_n \sim \text{Exp}(1)$ independently.  We let $N_i = \sum_{j = 1}^i E_j$ be the Poisson arrival times and let
\begin{align}
    S_i = \frac{N_i}{\nicefrac{d \nu}{d\mu}(X_i)}.
\end{align}
In order to construct a sample, we let $\jhat = \argmin_{1 \leq i \leq n} S_i$ and return $X_{\jhat}$.  We now analyze the total variation distance between the law of $X_{\jhat}$ and $\nu$.
\begin{lemma}
    Let $\mu, \nu$ be probability measures on $\cX$ and let $X_1, \dots, X_n \sim \mu$ be independent.  Let $\jhat$ be defined as above.  Letting $\nu_n$ be the law of $X_{\jhat}$, it holds that $\tv\left( \nu_n, \nu \right) \leq  \epsilon$ as long as
    \begin{align}
        n \geq 2 M \cdot \log\left( \nicefrac 3\epsilon \right) \quad \text{where} \quad \Cov_M\left( \nu \| \mu \right) \leq \nicefrac{\epsilon}{3}.
    \end{align}
\end{lemma}
\begin{proof}
    We adopt the proof strategy found in \citet[Theorem 1]{li2018strong}.  We form a coupling between the law of $Z_{\jhat}$ and $\nu$ as follows.  Extend the sequences $X_i, E_i$ to be infinite and let $K = \argmin_{i \geq 1} \frac{N_i}{\nicefrac{d \nu}{d\mu}(X_i)}$.  By the argument in \citet{li2018strong}, it holds that $Z_K \sim \nu$.  Thus it holds that 
    \begin{align}
        \tv\left( \nu_n, \nu \right) \leq \pp\left( K \neq \jhat \right) = \pp\left( K > n \right).
    \end{align}
    Letting $\Theta = \min_{i \geq 1} \frac{N_i}{\nicefrac{d \nu}{d\mu}(X_i)}$, we have by \citet{li2018strong} the the following hold: (i) $\Theta \sim \text{Exp}(1)$ and (ii) conditioned on $\Theta$ and $X_{K}$, it holds that $K - 1$ is a Poisson random variable with rate at most $\Theta \frac{d \nu}{d \mu}(X_K)$. Thus it holds that
    \begin{align}
        \pp\left( K > n \right) \leq \ee_{Y \sim \nu}\left[ \int_0^\infty e^{- \theta} \pp\left( \pois\left( \theta \frac{d \nu}{d \mu}(Y) \right) > n \right) d \theta\right].
    \end{align}
    Observe that poisson random variables with larger rate stochastically dominate those with smaller rate; thus it holds for any $M > 1$ that
    \begin{align}
        \ee_{Y \sim \nu}\left[ \int_0^\infty e^{- \theta} \pp\left( \pois\left( \theta \frac{d \nu}{d \mu}(Y) \right) > n \right) d \theta\right] &\leq \Cov_M\left( \nu \| \mu \right) + \int_0^\infty e^{- \theta} \pp\left( \pois\left( \theta M \right) > n \right) d \theta \\
        &\leq \Cov_M\left( \nu \| \mu \right) + \inf_{t > 0} e^{-t} + \pp\left( \pois\left( t M \right) > n \right) \\
        &\leq \nicefrac{2 \epsilon}{3} + \pp\left( \pois\left( M \cdot \log\left( \nicefrac{3}{\epsilon} \right) \right) > n \right).
    \end{align}
    Using standard Poisson tail bounds~\citep{vershynin2018high}, for $\lambda, t > 0$, it holds that
    \begin{align}
        \pp\left( \pois\left( \lambda \right) > (1 + t) \lambda \right) \leq e^{- \frac{t^2 \lambda}{2 (1 + t)}}.
    \end{align}
    Thus for $n \geq 2 M \log\left( \nicefrac{3}{\epsilon} \right)$, we have that $\pp\left( \pois\left( M \cdot \log\left( \nicefrac{3}{\epsilon} \right) \right) > n \right) \leq \nicefrac{\epsilon}{3}$.  The result follows.
\end{proof}
We can now prove the proposition by appealing to the above lemma and \Cref{lem:cov_fdiv}.
\begin{proof}[Proof of \Cref{prop:sampling}]
    By \Cref{lem:cov_fdiv}, it holds that
    \begin{align}
        \Cov_M\left( \nu \| \mu \right) \leq \frac{M \cdot \divf{\nu}{\mu}}{f(M)}.
    \end{align}
    Thus by setting $M$ such that $\frac{M \cdot \divf{\nu}{\mu}}{f(M)} \leq \nicefrac{\epsilon}{3}$, we have that
    \begin{align}
        n \gtrsim \log\left( \nicefrac 3\epsilon \right) \cdot \gamma_f\left( \frac{3 \cdot \divf{\nu}{\mu}}{\epsilon} \right)
    \end{align}
    samples suffice to ensure that $\tv\left( \nu_n, \nu \right) \leq \epsilon$.
\end{proof}

\end{document}